\documentclass[10pt,twocolumn,letterpaper]{article}

\usepackage{wacv}              % To produce the CAMERA-READY version
\usepackage{amssymb}
\usepackage[english]{babel}
\usepackage{amsthm}
\usepackage{amsmath,amsfonts}
\usepackage{mathtools, nccmath}
\usepackage{algorithmic}
\usepackage[ruled, noend, lined]{algorithm2e}
\usepackage{caption}
\usepackage{subcaption}
\usepackage{textcomp}
\usepackage{xcolor}
\usepackage{tabularx}
\usepackage{makecell}
\usepackage{booktabs}
\usepackage{graphicx}
\usepackage{glossaries}
\usepackage[numbers,sort & compress]{natbib}
\newcommand{\FW}{\texttt{DARDA}\xspace}

\usepackage[many]{tcolorbox}
\usepackage{lipsum}

\definecolor{titlebg}{RGB}{100,22,72}
\definecolor{introbg}{RGB}{115,70,150}

\usepackage[pagebackref,breaklinks,colorlinks]{hyperref}
\usepackage[capitalize]{cleveref}
\crefname{section}{Sec.}{Secs.}
\Crefname{section}{Section}{Sections}
\Crefname{table}{Table}{Tables}
\crefname{table}{Tab.}{Tabs.}

\def\wacvPaperID{960} % *** Enter the WACV Paper ID here
\def\confName{WACV}
\def\confYear{2025}
\newtcolorbox{usecase}[1][]{
  breakable,
  enhanced,
  arc=0pt,
  outer arc=0pt,
  colframe=titlebg,
  colback=titlebg!05,
  overlay unbroken and first={
    \node[
      draw=titlebg,
      fill=titlebg,
      rotate=0,
      anchor=north west,
      text=white,
      font=\bfseries
    ]
    at (frame.north west)  
    {#1};
  }
}

\newtcolorbox{mission}[1][]{
  breakable,
  enhanced,
  arc=0pt,
  outer arc=0pt,
  colframe=introbg,
  colback=introbg!05,
  overlay unbroken and first={
    \node[
      draw=introbg,
      fill=introbg,
      rotate=0,
      anchor=north west,
      text=white,
      font=\bfseries
    ]
    at (frame.north west)  
    {#1};
  }
}

\newacronym{dl}{DL}{Deep Learning}
\newacronym{amc}{AMC}{Automatic Modulation Classification}
\newacronym{dnn}{DNN}{Deep Neural Network}
\newacronym{nn}{NN}{ Neural Network}
\newacronym{bn}{BN}{ Batch Normalization}
\newacronym{cnn}{CNN}{Convolutional Neural Network}
\newacronym{gap}{GAP}{Global Average Pooling}
\newacronym{fc}{FC}{Fully Connected}
\newacronym{snr}{SNR}{signal-to-noise ratio}
\newacronym{qos}{QOS}{Quality of Service}
\newacronym{5g}{5G}{Fifth Generation}
\newacronym{tl}{TL}{Transfer Learning}
\newacronym{fls}{FSL}{Few Shot Learning}
\newacronym{dg}{DG}{Domain Generalization}
\newacronym{tta}{TTA}{Test Time Adaptation}
\newacronym{iid}{IID}{independent and identically distributed}
\newacronym{uda}{UDA}{Unsupervised Domain Adaptation}
\newacronym{mse}{MSE}{mean squared error}

\newacronym{uav}{UAV}{unmanned autonomous vehicle}
\newacronym{xr}{XR}{extended reality}
\newacronym{vr}{VR}{virtual reality}
\newacronym{ar}{AR}{augmented reality}
\newacronym{cv}{CV}{computer vision}

\newacronym{ft}{FT}{Fine Tuning}
\newacronym{sfda}{SFDA}{Source-Free Domain Adaptation}
\newacronym{mac}{MAC}{multiply and accumulate}
\newtheorem{prop}{Proposition}
\usepackage{graphicx}
\usepackage{amsmath}
\usepackage{amssymb}
\usepackage{booktabs}

\usepackage[pagebackref,breaklinks,colorlinks]{hyperref}

\usepackage[capitalize]{cleveref}
\crefname{section}{Sec.}{Secs.}
\Crefname{section}{Section}{Sections}
\Crefname{table}{Table}{Tables}
\crefname{table}{Tab.}{Tabs.}

\def\wacvPaperID{960} % *** Enter the WACV Paper ID here
\def\confName{WACV}
\def\confYear{2025}

\usepackage{xr}
\makeatletter
\newcommand*{\addFileDependency}[1]{% argument=file name and extension
  \typeout{(#1)}
  \@addtofilelist{#1}
  \IfFileExists{#1}{}{\typeout{No file #1.}}
}
\makeatother

\newcommand*{\myexternaldocument}[1]{%
    \externaldocument{#1}%
    \addFileDependency{#1.tex}%
    \addFileDependency{#1.aux}%
}

\myexternaldocument{supplementary}

\begin{document}

\title{\FW: Domain-Aware Real-Time Dynamic Neural Network Adaptation}

\author{Shahriar Rifat\\
Northeastern University\\
United States \\ 
{\tt\small rifat.s@northeastern.edu}
\and
Jonathan Ashdown\\
Air Force Research Laboratory\\
United States\\
{\tt\small jonathan.ashdown@us.af.mil}
\and
Francesco Restuccia\\
Northeastern University\\
United States\\
{\tt\small f.restuccia@northeastern.edu}
}
\maketitle

%\thanks{\noindent Approved for Public Release; Distribution Unlimited: AFRL-2023-XXX} 
\glsresetall

%%%%%%%%% ABSTRACT
\begin{abstract}
\gls{tta} has emerged as a practical solution to mitigate the performance degradation of \Glspl{dnn} in the presence of corruption/ noise affecting inputs. Existing approaches in \gls{tta} continuously adapt the \gls{dnn}, leading to excessive resource consumption and performance degradation due to accumulation of error stemming from lack of supervision. In this work, we propose Domain-Aware Real-Time Dynamic Adaptation (\FW) to address such issues. Our key approach is to proactively learn latent representations of some corruption types, each one associated with a sub-network state tailored to correctly classify inputs affected by that corruption. After deployment, \FW adapts the \gls{dnn} to \textit{previously unseen} corruptions in an \emph{unsupervised fashion} by (i) estimating the latent representation of the ongoing corruption; (ii) selecting the sub-network whose associated corruption is the closest in the latent space to the ongoing corruption; and (iii) adapting \gls{dnn} state, so that its representation matches the ongoing corruption. This way, \FW is more resource-efficient and can swiftly adapt to new distributions caused by different corruptions without requiring a large variety of input data. Through experiments with two popular mobile edge devices -- Raspberry Pi and NVIDIA Jetson Nano --  we show that \FW reduces energy consumption and average cache memory footprint respectively by $1.74\times$ and $2.64\times$ with respect to the state of the art, while increasing the performance by $10.4\%$, $5.7\%$ and $4.4\%$ on CIFAR-10, CIFAR-100 and TinyImagenet.
\end{abstract}

%%%%%%%%% BODY TEXT
\vspace{-.6 cm}
\section{Introduction} \label{sec:intro}

% State-of-the-art mobile devices need to perform complex tasks often based on \glspl{dnn}. For example, object detection \cite{wu2020recent} and semantic segmentation \cite{mo2022review} are needed by \glspl{uav} for surveillance purposes, as well as to avoid obstacles during navigation \cite{wang2020uav,fraga2019review}  and to continuously build detailed 3D maps \cite{ChallengesSelfDriving}. On the other hand, modern \glspl{dnn} have computational requirements that go beyond the computational capabilities of existing mobile devices. While mobile-specific \glspl{dnn} such as MobileNet~\cite{sandler2018mobilenetv2} and MnasNet~\cite{tan2019mnasnet} are significantly less complex, they come at the detriment of accuracy. For example, MobileNet loses up to 6.4\% in accuracy compared to the ResNet-152~\cite{he2016deep}.  Despite experiencing performance degradation, there is a growing interest in the deployment of DNN directly on edge devices to eliminate additional communication latency and mitigate privacy risks associated with computational offloading \cite{wu2021pecam,guo2021mistify}.

% Conversely, \textit{mobile edge computing} tackles the problem by offloading the computation of large \gls{dnn} to powerful edge computers to maintain accuracy while preserving the energy consumption of the mobile device \cite{mao2017survey}. In short, edge servers are tasked to (i) wirelessly receive the input data, (ii) perform the \gls{dnn} computation and (iii) send back the result to the mobile device. \smallskip

Traditional mobile edge computing scenarios assume that the inputs of \glspl{dnn} are received uncorrupted. However, in many real-life scenarios, sudden and unexpected corruptions (e.g., snowy or foggy conditions) can cause a drastic change in data distribution, consequently causing  performance loss \cite{sayed2012learning,shawabka2020exposing}. For example, a semantic segmentation \gls{dnn} trained with data collected in normal weather conditions has been shown to exhibit a performance loss of more than 30\% when tested in snowy conditions \cite{sakaridis2021acdc}, while an image classification \gls{dnn} can experience a similar decrease in the case of reduced lighting conditions \cite{hendrycks2019benchmarking}.

\begin{figure}[h!]
	\centering
	\includegraphics[width=\columnwidth]{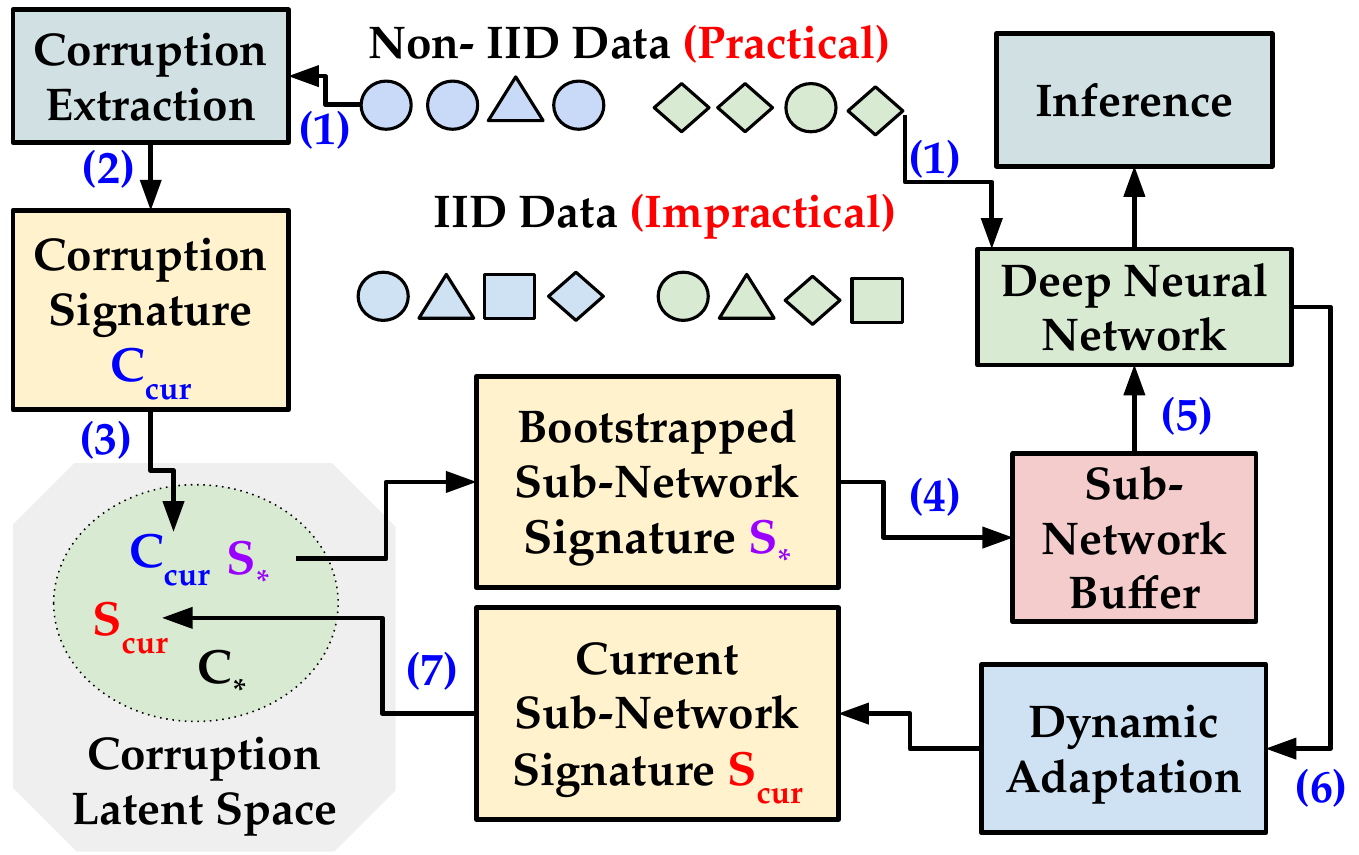} 
	\caption{Overview of the proposed \FW framework.} \vspace{-.2 cm}
	\label{fig:concept}

\end{figure}
Test-time Adaptation (\gls{tta}) tackles this issue by adapting the \gls{dnn} with unlabeled test data in an online manner, thus handling distributional shifts in real time. Existing \gls{tta} methods lose performance when encountering continuously changing distributions with highly correlated input samples \cite{wang2020tent,wang2022continual}. This assumption is true in many real-world scenarios. For example, an \gls{uav} monitoring an outdoor environment will likely encounter similar classes as video feeds are very likely to be highly correlated when considering limited time spans. Continuous adaptation of an edge deployed \gls{dnn} to such a challenging yet practical scenario causes many adaptation methods to fail. Moreover, existing methods lack awareness of when the domain shift happens, thus they continuously fine-tune the \gls{dnn} even if there is no shift in data distribution. However, in a real-life deployment scenario, certain data distribution might persist for a certain period of time (e.g., a bright sunny day). This imposes unnecessary burden on energy consumption and cache memory --  without yielding performance improvements.

To address the critical issues defined above, we propose a new  framework named \textit{Domain-Aware Real-Time Dynamic Adaptation} (\FW), shown in \cref{fig:concept}. We now provide a step-by-step walk-through of \FW. First, the data stream can shift due to some corruption which is also correlated in label space (distribution of labels are not uniform) (step 1). The corruption process is then detected and extracted (step 2), resulting in a latent representation (step 3), which we call \textit{signature} for brevity. Then, the corruption signature $ C_{cor}$ is assigned to the closest \textit{corruption centroid} $C_{*}$. Each centroid is learned during training and represents a known corruption type. Moreover, it is associated with a ``bootstrapped'' sub-network of the main \gls{dnn} that is specifically tailored to the specific corruption. The sub-network signature $S_{*}$  is then used to retrieve the actual sub-network structure and weight (step 4), which is then immediately plugged into the DNN (step 5). Next, the \gls{dnn} is updated to match the type of ongoing corruption (\textit{not seen during the subnetwork and corruption signature training phase}) (step 6) by ``moving'' the current subnetwork signature $S_{cur}$ closer to the actual ongoing corruption $C_{cur}$ (step 7).

% \noindent \textbf{Key Advantages.}~\FW avoids continuously adapting the \gls{dnn} to the current corruption type, and instead learns a  \textit{corruption latent space} that allows it to quickly match the ongoing corruption to certain extent. Second, \FW is able to adapt to \emph{unseen} corruptions with a \textit{limited} number of  \textit{unlabeled} inputs. Third, as shown in \cref{fig:forgetting}, \FW does not incur catastrophic forgetting as it is able to keep the accuracy within 1\% given by the original \gls{dnn}, while other methods decrease accuracy by up to about 50\%. \vspace{-0.2cm}

% \vspace{-.2cm}
% \begin{figure}[h]
% 	\centering
% 	\includegraphics[width=\columnwidth]{fig/Covariate & Prior Shift.pdf} 
% 	\caption{Example of Label Distribution \& Corruption Shift.}
% 	\label{fig:prior & covrarite shift}

% \end{figure}
\vspace{-0.2cm}
\subsection*{\textbf{Summary of Novel Contributions}}
$\bullet$ We propose Domain-Aware Rapid Dynamic Adaptation (\FW) to seamlessly and effectively adapt in real-time state-of-the-art \gls{dnn} to unseen corruptions (Section \cref{sec:fw}). The key innovation of \FW is a brand-new approach to learn a latent space putting together the corruption process and the state of the \gls{dnn}, which is done through a \textit{corruption extractor} (Section  \cref{sec:corruption-extractor}), a corruption encoder (Section  \cref{sec:corruption-encoder}) and  a \textit{sub-network encoder} (Section \cref{sec:model-encoder}), which together make \FW able to seamlessly adapt the \gls{dnn} with \textit{unlabeled} inputs. The implementation is available at: \href{https://github.com/shahriar-rifat/DARDA.git}{darda repository}. \smallskip   

$\bullet$ We prototype \FW and evaluate its performance against five state of the art approaches, namely BN  \cite{nado2020evaluating}, TENT \cite{wang2020tent}, CoTTA \cite{wang2022continual}, NOTE  \cite{gong2022note}, RoTTA \cite{yuan2023robust} on the ResNet-56 \gls{dnn} trained with the CIFAR-10 and CIFAR-100 datasets augmented with same known corruption types that are considered to be known prior to deployment. We show that \FW improves the performance by 10.4\% and 5.7\%  on CIFAR-10 and CIFAR-100 respectively, while performing 29\% less forward computation and 77\% less backward passes during the adaptation process and with only 16.57\% of additional memory. \smallskip

$\bullet$ We implement \FW on Jetson-Nano and Raspberry Pi 5, commonly used to exhibit efficiency of edge-deployed DNNs \cite{wen2023adaptivenet,han2021legodnn}. Experiments show that \FW handle distribution shifts while being $1.74\times$ more energy efficient than the best-performing state of the art \gls{tta} algorithm. For adaptation and inference task \FW takes $7.3\times$ less time per sample while it occupies $2.64\times$ less cache memory. \vspace{-0.3cm}

\section{Related Work and Existing Issues}

Some existing works \cite{schneider2020improving,gong2022note} on \gls{tta} have provided empirical evidence of performance improvement by only re-estimating the normalization statistics of \gls{bn} layers from test data. The absence of supervision is typically covered by two unsupervised forms of losses. Firstly, a line of work \cite{wang2021tent,niu2022efficient, goyal2022test} minimizes the entropy of the predictions over a batch of data to prevent the collapse of a trivial solution. Invariance regularization-based \gls{tta} algorithms perform some data augmentation (e.g. rotation \cite{wang2022continual}, adversarial perturbation \cite{nguyen2023tipi}) on test data during inference. The inconsistency of the prediction of \gls{dnn} on different augmented test data is leveraged as an unsupervised loss function to update the learnable parameters during inference. The proposed \FW framework uses cross-modal learning to acquire a shared representation space between the corruption space and the \gls{dnn} space \cite{radford2021learning, zhu2022multimodal}. However, to our knowledge, none of the existing research addresses cross-modal learning between the corruption process and the state of the \gls{dnn} model.  Next, we discuss into some practical limitations of \gls{tta} in edge vision application. \smallskip

\noindent \textbf{Excessive Resource Consumption.}~To improve performance, existing \gls{tta} approaches typically involve continuous adaptation even with uncorrupted input samples, thus imposing a heavy burden on edge resources. Ideally, adaptation should be only performed upon changes in the corruption process, thereby conserving constrained resources such as energy and processing power at the edge. Despite the potential benefits, current methods have yet to explore this direction. \cref{fig: motiv 1} shows a significant increase in resources between inference-only and existing \gls{tta} approaches in Jetson Nano. Specifically, the energy consumption increases by up to $11.9\times$ , along with a $6.8\times$ increase in CPU latency and a $3.1\times$ increase in GPU latency. 

\vspace{-.2cm}
\begin{figure} [h]
        \centering
        \begin{subfigure}[b]{0.21\textwidth}
                \centering
                \includegraphics[width=\textwidth]{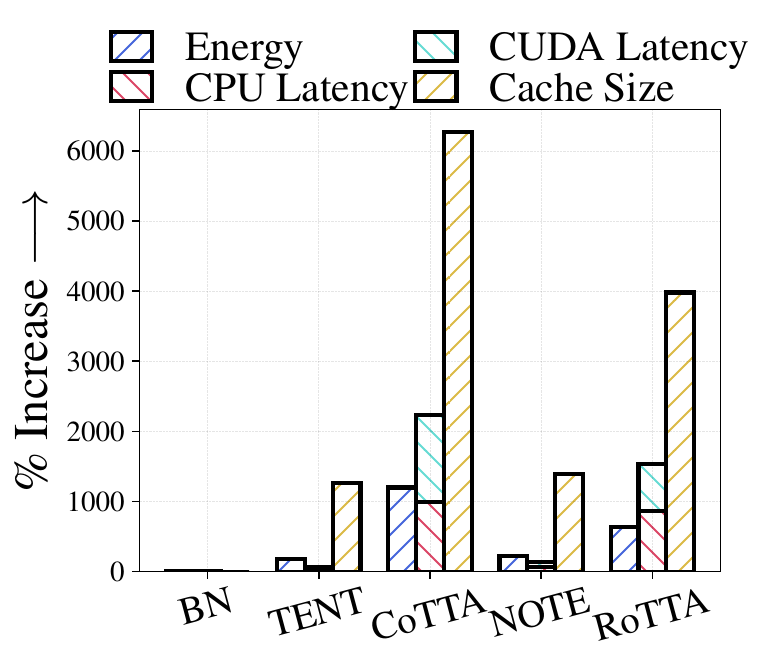}
                \caption{Increase in resource consumption with respect to inference-only execution}
                \label{fig: motiv 1}
        \end{subfigure}
        \begin{subfigure}[b]{.21\textwidth}
                \centering
                \includegraphics[width=\textwidth]{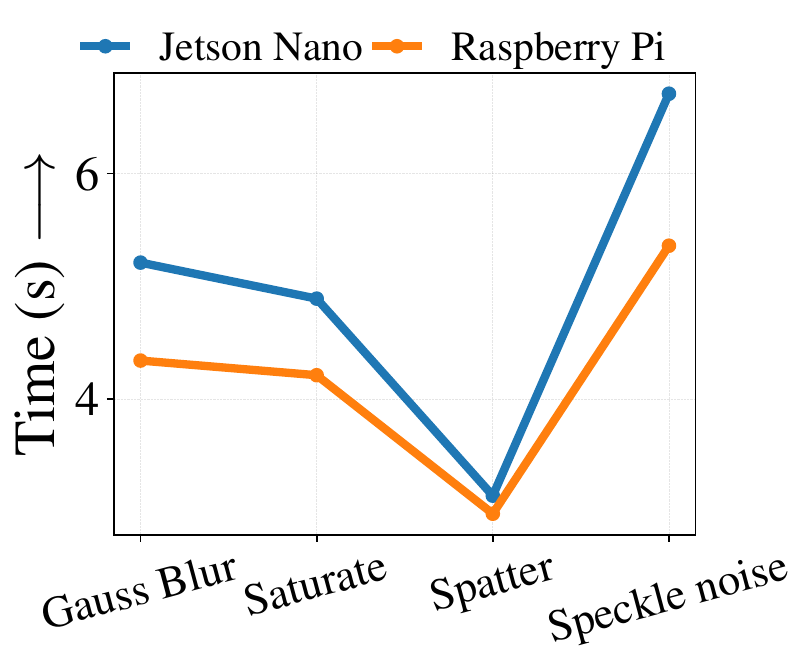}
                \caption{Adaptation latency for Jetson Nano and Raspberry Pi for different corruption types}
                \label{fig: motiv 2}
        \end{subfigure}
        \caption{Current issues of Test-time Adaptation. }\label{fig:edge efficiency}
        
\end{figure}
\vspace{-.3 cm}

\noindent \textbf{Higher Average Cache Usage.}~Mobile edge devices have limited memory size. As such, it is compelling to ensure \gls{tta} uses minimal memory footprint. Specifically, during inference, various blocks of a \gls{dnn} are executed sequentially, and the cache memory usage at any given time is bounded by the data size of the activation map of the block being computed. However, as the \gls{tta} dynamically updates the \gls{dnn}, local gradients of learnable parameters with respect to intermediate activation are stored, a quantity that scales with both the \gls{dnn} size and the number of learnable parameters. From \cref{fig: motiv 1}, we observe that the average cache usage increases by up to $21.7\times$ compared to a \gls{dnn} deployed for inference only. \smallskip

\noindent \textbf{Dependence on Sample Diversity}.~To be effective in real-world scenarios, the dynamic adaptation of a \gls{dnn} needs to happen in a rapid fashion. However, Figure \cref{fig: motiv 2} shows that existing state of the art work \cite{yuan2023robust} takes significant time to restore the performance of the original \gls{dnn} when the noise condition keeps changing. The reason behind such behavior lies in  the inherent dependency on sample diversity of \cite{yuan2023robust}. In other words, if diverse samples are not observed with a new corruption process, the work \cite{yuan2023robust} cannot achieve its optimum performance gains.\vspace{-0.2cm}

% The proposed \FW framework takes a new direction and uses cross-modal learning to acquire a shared representation space between the corruption space and the \gls{dnn} space \cite{radford2021learning, zhu2022multimodal}. However, to the best of our knowledge, none of the existing research addresses cross-modal learning between the corruption process and the state of the \gls{dnn} model. 

\section{Problem Statement} \label{motivation}

We define $d_s$ as the number of available learning domains, each characterizing a different imperfection and/or corruption type. We further define the related set of $d_s$ datasets where $x_i^d$ and $y_i^d$ indicates $i^{th}$ data sample and label from domain d respectively, as 

\begin{equation}\mathcal{D}^{d} = \left \{ \left( x_i^{d}, y_i^{d}\right)\right \}_{i=1}^{n_{d}} \mbox{, with } 0 \le d < d_s. 
\end{equation} 

%  \begin{figure}[h]
% 	\centering
% 	\includegraphics[width=0.9\columnwidth]{fig/DNN-Sub-Network.pdf}
% 	\caption{Domain-Adapted Sub-Networks.}

Each dataset $\mathcal{D}^{d} $ is composed of $n_{d_s}$ \gls{iid} samples characterized by some probability distribution $\mathit{P}^d(X,Y) $ where X and Y are random variables of input and output respectively. We assume a \gls{dnn} has been trained on an uncorrupted dataset and $d_s$ sub-networks $f(.; \theta_d)$ are created so that (i) their architecture includes the batch normalization layer and the dense layers of the \gls{dnn}; (ii) their weights $\theta_d$ are obtained by fine-tuning each sub-network to each specific domain. We assume continuous and correlated (thus, non-\gls{iid}) data flow to the \gls{dnn} in real time, coming from $d_u$ unknown domain datasets $\mathcal{D}^u \mbox{, with } 0 \le u < d_u$. By unknown we mean $\mathit{P}^u(X,Y) \neq \mathit{P}^d(X,Y)$, for all ($0 \le d < d_s$ , $0 \le u < d_u$). We define a \textit{domain latent space} $\mathcal{O} \subset \mathbb{R}^o$, where $o$ is the dimension of the latent space.  Our goal is to (i) sense when the data flow has changed domain from the current domain $d$ to the unknown domain $u$; (ii) infer  domain $t$ that is closest to $u$ in the latent space; (ii) select the related fine-tuned sub-network $f(\cdot; \theta_t)$ to quickly recover performance, and (iii) adapt $f(\cdot;\theta_t)$ so as to find optimal $f^{*}(.; \theta^*)$ such that:

\vspace{-.3cm}
\begin{equation}
   \theta_t^* =  \operatorname*{arg\ min}_{\theta_t} \frac{1}{n_u}\sum_{i=1}^{n_u}\mathcal{L}\left\{f\left(x_i^{u}; \theta_t \right); y_i^{u}\right\}
\end{equation}
%\end{mission}

\noindent where $n_u$ is a given number of samples in the unknown domain. Such samples are assumed to be available sequentially and the distribution of labels is different from the current domain's, i.e., $\mathit{P}^u(Y) \neq \mathit{P}^d(Y)$. Notice that ground-truth labels $y_i^{u}$ are usually not available in real-world settings and are only used for performance evaluation.

\section{Description of \FW Framework } \label{sec:fw}
\vspace{-.1cm}
The main components of \FW are a corruption extractor (Section \cref{sec:corruption-extractor}), a corruption encoder (Section \cref{sec:corruption-encoder}) and a model encoder (Section \cref{sec:model-encoder}), a new corruption-aware memory bank (Section \cref{sec:memory-bank}), new batch normalization scheme (Section \cref{sec:batch-normalization}) and a new real-time adaptation module (Section \cref{sec:real-time-adaptation}).

 \begin{figure}[h]
	\centering
	\includegraphics[width=0.9\columnwidth]{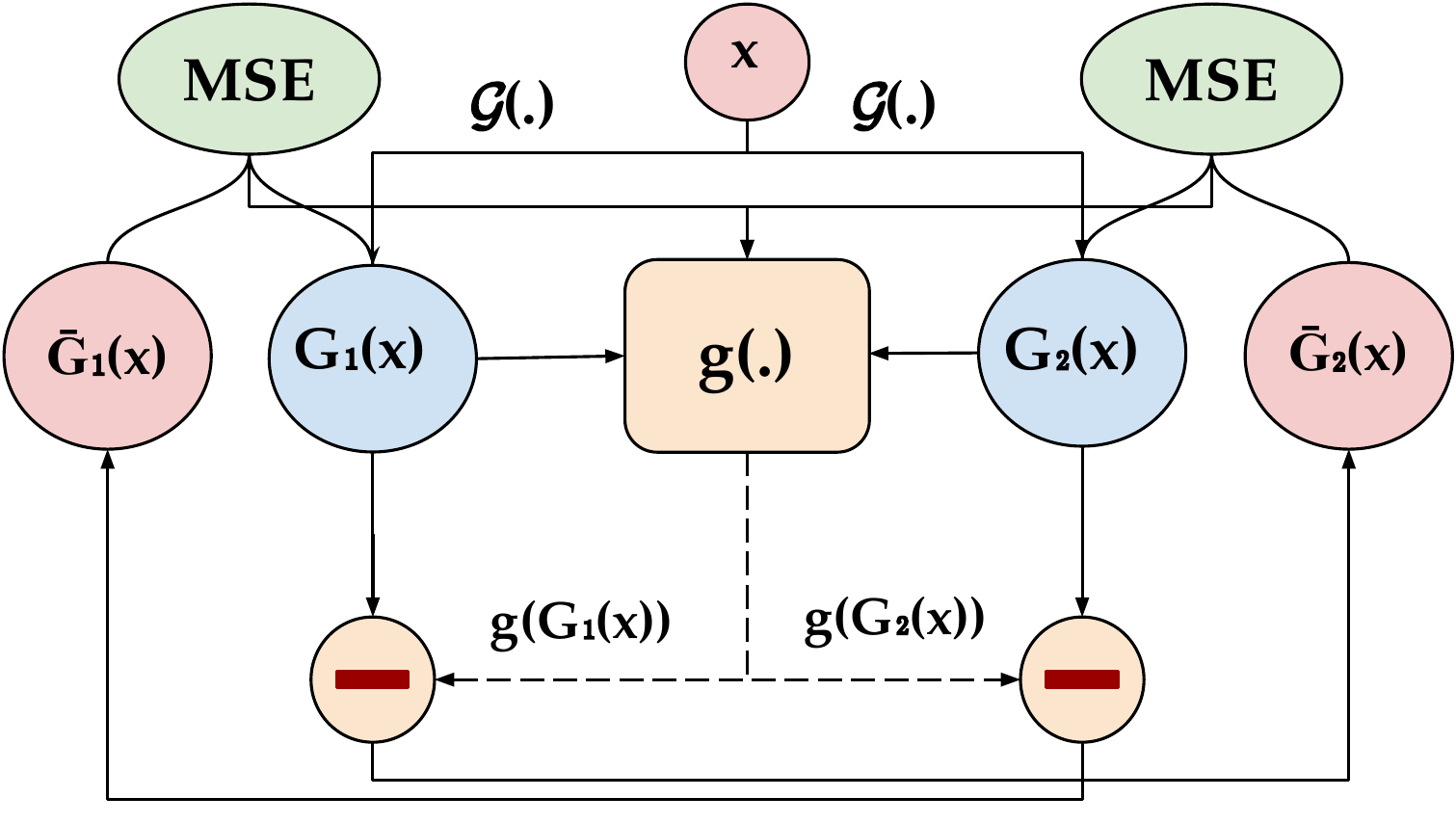}
	\caption{Proposed Corruption Extractor.}
	\label{fig:noise_residual}
\vspace{-0.3cm}
\end{figure}
\subsection{\FW Corruption Extractor}\label{sec:corruption-extractor}

\cref{fig:noise_residual} shows our proposed corruption extraction approach. Our key intuition is that features related to corruption and semantic features for inference are tightly intertwined. Since decoupling these features is difficult without corresponding clean samples, we design a process to decouple corruption features without corresponding clean sample. Specifically, we learn the corruption features by mapping corrupted data to a different corrupted version of the same data \cite{huang2021neighbor2neighbor, mansour2023zero}. For a given corrupted data $x$ we downsample it through two convolution kernels with static filters $G_1(.) = [[0,0.5],[0.5,0]]$ and $G_2(.) = [[0.5, 0],[0,0.5]]$ to generate two downsampled version of the corrupted data. From the first downsampled corrupted data $\mathbf{G}_{1}(x)$, we try to create an exact copy of the other downsampled data $\mathbf{G}_{2}(x)$ by subtracting some residual information learned by passing $\mathbf{G}_{1}(x)$ through the corruption extractor $g(\cdot)$. We denote this predicted copy as $\mathbf{\tilde{G}}_{2}(x)$. Similarly, we compute $\mathbf{\tilde{G}}_{1}(x)$ from $\mathbf{G}_{2}(x)$. The mapping functions are as follows: 

% We use a \gls{cnn} $\mathbf{f_{\theta_n}(.)}$ with parameters $\theta_n$ to learn some residual value which can be discarded to map one downsampled version of corrupted data $\mathbf{G}_{1}(x)$ to another one as $\mathbf{G}_{2}(x)$ as shown in Figure \cref{fig:noise_residual} by subtracting the noise residual using the following equations:  

\begin{equation} \label{eqn:g(x)}
    \tilde{\mathbf{G}_{2}}(x) = \mathbf{G}_{1}(x) - g(\mathbf{G}_{1}(x))
\end{equation}
\vspace{-.5cm}
\begin{align}
    \tilde{\mathbf{G}_{1}}(x) = \mathbf{G}_{2}(x) - g(\mathbf{G}_{2}(x))
\end{align}

In \cref{fig:noise_residual}, the extracted residual is denoted by dotted lines. The parameters of $g(\cdot)$ can be optimized by minimizing the following loss function, which is the loss \gls{mse} indicated in \cref{fig:noise_residual}.

\begin{equation} \label{eqn:Ln}
\begin{split}
   \mathcal{L}_{N} = & \sum_{c=1}^{d_s}\sum_{i=1}^{n_{d_s}}\frac{1}{2}(\left \| \tilde{\mathbf{G}_{2}}(x_{c,i})-\mathbf{G}_{2}(x_{c,i}) \right \|_{2}^{2} + \\ \vspace{-.5cm}
  & \left \| \tilde{\mathbf{G}_{1}}(x_{c,i})-\mathbf{G}_{1}(x_{c,i}) \right \|_{2}^{2})
   \end{split}
\end{equation}
The intuition behind our approach is that the pixel values of the uncorrupted data in close proximity are usually highly correlated. Therefore, two downsampled versions of the data would be almost the same since they were generated by averaging values in close proximity. The corruption process breaks this correlation. Thus, our extractor $g(\cdot)$ learns to extract a representation of the corruption to generate uncorrupted data. We design this loss function in \cref{eqn:Ln} to create an opposite  dynamics that enable us to extract information about the corruption process. We have theoretically proven through \cref{prop:prop_1} that, for additive noise our proposed approach indeed learns to extract information about the corruption. In \cref{table:Noise Type Comparison}, it is empirically verified that the proposed corruption extractor is useful for different kinds of corruption in general.

\subsection{\FW Corruption Encoder}\label{sec:corruption-encoder}

We use the corruption-related features to detect a corruption shift in real time. Specifically, we use a corruption encoder $ h(\cdot)$ to encode corruption information of the input data from the known corruption types into a projection in the corruption latent space. While generating the latent space, we  ensure that samples from the same corruption distribution are grouped together in that latent space and samples from different corruption distribution are located distant from each other. For the N samples $\left \{ C^i, D^i\right \}_{i=1}^N $ in a training data batch, we define $\mathcal{C}$ as the set of each corruption projection $C^i$ in the latent space. We also define as $D^i$ as the corruption label for projection $C^i$, and as $\mathcal{D}$ the corresponding set. We define the following supervised contrastive loss function for a batch of training data:

\begin{equation} \label{eqn:supcon}
    \mathcal{L}_D(\mathcal{C}, \mathcal{D}) = \sum_{i=1}^{2 \cdot N}\mathcal{L}_D^{i}(\mathcal{C},\mathcal{D})
\end{equation}
where $\mathcal{L}_D^{i}(\mathcal{C},\mathcal{D})$ is defined as

\begin{equation} \label{eqn:supcon_2}
\begin{split}
\mathcal{L}_D^{i}(\mathcal{C},\mathcal{D}) = &
    \frac{-1}{2\cdot n_{d_s}-1}\sum_{j=1}^{2\cdot N}\mathbf{1}_{(i\neq j)\&(D^{i}\neq D^{j})} \\ \times & \log\frac{exp(C^i\mathbf{.}C^j/\tau)}{\sum_{k=1}^{2\cdot N}\mathbf{1}_{(k\neq i)}exp(C^i\mathbf{.}C^k/\tau)}   
    \end{split}
\end{equation}
where $N$ represents the total number of samples in the batch and $\tau$ is a scaling parameter. While training the corruption encoder, we generate a soft augmentation (random rotations and flips) from each data sample to have more samples from each noise classes. This way, our training batch size becomes $2 \cdot N$. For each sample $i$ in our training batch, we calculate its contrastive loss  using Eq. \cref{eqn:supcon_2}. Here, the numerator enforces cosine similarity between similar corruption types and the denominator penalizes high similarity between projections which are from different corruption class. Thus, similar corruption projections are positioned closer and dissimilar ones are positioned far apart.  We jointly train the corruption extractor $g(.)$ and corruption encoder $h(.)$  by minimizing the loss function:

\begin{equation}\label{eq:training_corruption_encoder}
    \mathcal{L}=\mathcal{L}_D + \lambda_{e}\cdot \mathcal{L}_N
\end{equation} 

where $\lambda_{e}$ is a constant which does not impact performance yet makes the convergence of $g(.)$ and $h(.)$ faster. The  training process is described in \cref{alg:corruption-encoder}.

\subsection{\FW Sub-Network Encoder}\label{sec:model-encoder}

% As the trained noise-encoder is demonstrated to be capable enough to project samples from an unknown distribution closer to the most similar known corruption distribution,  we can simply assign a unique number to each model corresponding to the corruption distribution it is optimized for, and consequently retrieve the model from the buffer to instantly have some stable performance gain as models trained on similar distribution gives performance boast than the model trained on different corruption distribution. We can also fine-tune the model using the oncoming samples it observes during its inference time using existing methods like entropy minimization\cite{wang2020tent}.  

% % However, as discussed previously, such loss functions are not well calibrated and do not guarantee stable performance gain.
To guide the adaptation of the sub-network,  we need to obtain a ``fingerprint'' of the current sub-network, whose state space is by definition continuous and infinite. We address this issue by creating a set of unique fingerprints $F^1 \cdots F^{d_s}$ of each sub-network by feeding a fixed Gaussian noise into the \gls{dnn} and consider its output response vector as the fingerprint of the sub-network.  Our intuition is that since a \gls{dnn} works as a non-linear function approximator, it will produce a different output for the same input with different parameters. We generate a signature $S^{d}$ of each sub-network from each fingerprint $F^d$ as $S^{d} = \mathcal{S}(F^d; \psi)$ where $\mathcal{S}: F \rightarrow \mathbb{R}^o$ which is a shallow neural network parameterized by $\psi$ that maps the fingerprint to the corruption latent space. The $\mathcal{S}$ encoder is trained so as to minimize the following loss function, which maximizes the cosine similarity between the latent space projections: 
\begin{equation}
    \mathcal{L}_{CM} = \sum_{i=1}^{d_s}\sum_{j=1}^{d_s} \mathbf{1}_{i=j} \left\{exp(-\,S^i\,\mathbf{.}\, C^j)\right\}
\end{equation} 

Here, $\mathbf{1}(.)$ is the indicator function. We use a regularization term in addition to $\mathcal{L}_{CM}$ to distribute encoder $\mathcal{S}$'s projection in regions from where sub-network's projections would produce well-performing sub networks. The measured cosine similarity between a sub-network signature $S^i$ and a corruption signature $C_j$ is converted into probability distribution $\pi_{ij}$ using: 
\begin{equation}
    \pi_{ij} = \sigma \left(\frac{S^i \cdot C^j}{\sum_{k=1}^{d_s} S^k \cdot  C^k}\right), i,j \in [0, d_s]
\end{equation}
where $\sigma$ is the softmax function. If $a_{ij}$ indicates accuracy of sub-network $i$ in corruption domain $j$ we can generate a probability distribution $\alpha_{ij}$ such as
\begin{equation}
    \alpha_{ij} =\sigma \left( \frac{\log 1/(1 - a_{ij}) }{\sum_{k=1}^{d_s}\log 1/(1-a_{ik}) }\right), i,j \in [0, d_s]
\end{equation}
% to generate projections that have affinity with other corruption domains where it can perform well 
We can calibrate the sub-network encoder to to generate projections that have affinity with other corruption domains where it can perform well by minimizing the KL divergence between $\pi_{ij}$ and $\alpha_{ij}$. The regularization term $\mathcal{L}_{r}$ and the loss $\mathcal{L}_{m}$ that the sub-network encoder is trained on are
\begin{equation}
    \mathcal{L}_{r} = \sum_{i=1}^{d_s}\sum_{j=1}^{d_s} \log \frac{\pi_{ij}}{\alpha_{ij}}
\end{equation}
\begin{equation}
    \mathcal{L}_m = \mathcal{L}_{CM} \, + \, \lambda_r\,\mathcal{L}_r,
\end{equation}
where $0 < \lambda_r < 1$ is a regularization parameter. 

% In a well calibrated latent space (prediction closely quantifies the actual probability), the similarity in the space can be used as a surrogate of its performance and provide reliable gradient flow to fine tune the model in unsupervised manner. To do so, we train the model encoder to generate its projection a trained sub networks closest to the corruption it was trained on. At the same time we want our model encoder not to be too overconfident and make it projection probability distributed in accordance to its performance in the latent space. 

\subsection{Corruption-Aware Memory Bank}\label{sec:memory-bank}

In practical scenarios, the distribution of labels differs from the actual label distribution. Importantly, while during training, the \gls{dnn} is given input data with \gls{iid} labels, in real-world scenarios sequential data is highly correlated while other classes are very scarce at a particular time. Adaptation to this unreliable label distribution leads to substantial performance loss in traditional approaches, as shown in  Section \cref{sec:results}. To address this problem, we need to have a stable snapshot of the ongoing corruption at inference time. Thus, we create and maintain a memory bank $\mathcal{M}$ with $\mathcal{N}$ slots to store samples. We construct the memory bank in a label-balanced manner. Recalling that $Y$ is the set of labels, for each class $y\in Y$, we store $\frac{\mathcal{N}}{\mid Y\mid}$ number of incoming test samples. As we do not have labeled data, the labels are inferred from the prediction $\hat{y}$ of the model. However, sampling based on prediction of continuously adapted model leads to error accumulation \cite{yuan2023robust}.

To solve error accumulation, existing methods \cite{yuan2023robust,wang2022continual} resort to inference with multiple \glspl{dnn} by feeding different augmented views of test samples to them. However, this requires additional computation and memory for multiple inference. Although sensing the corruption and bootstrapping with proper sub-network signature leads to reliable memory bank construction, we store the samples that are only representative of the ongoing corruption. For each incoming test samples we predict its class label, and store it in the memory bank if we have room for that particular class and if it is highly representative of the ongoing corruption. The process of memory bank construction is described in \cref{alg:memory bank}.

\begin{figure*}[t]
        \centering
        \begin{subfigure}[b]{0.32\textwidth}
                \includegraphics[width=\linewidth]{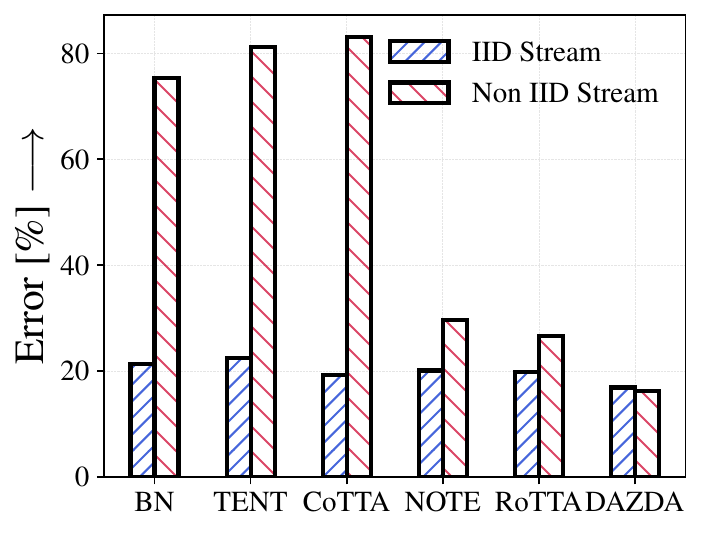}
                \caption{CIFAR-10 }
                \label{fig: cifar10_performance}
        \end{subfigure}%
        \begin{subfigure}[b]{0.32\textwidth}
                \includegraphics[width=\linewidth]{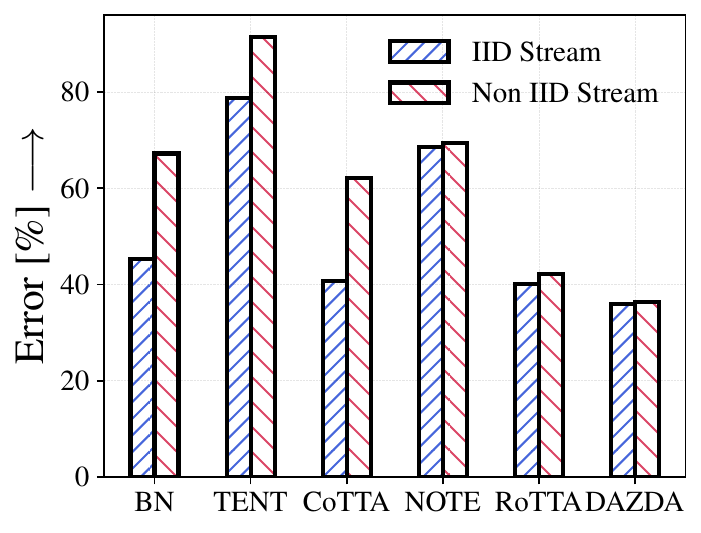}
                \caption{CIFAR-100 }
                \label{fig: cifar100_performance}
        \end{subfigure}%
        \begin{subfigure}[b]{0.32\textwidth}
                \includegraphics[width=\linewidth]{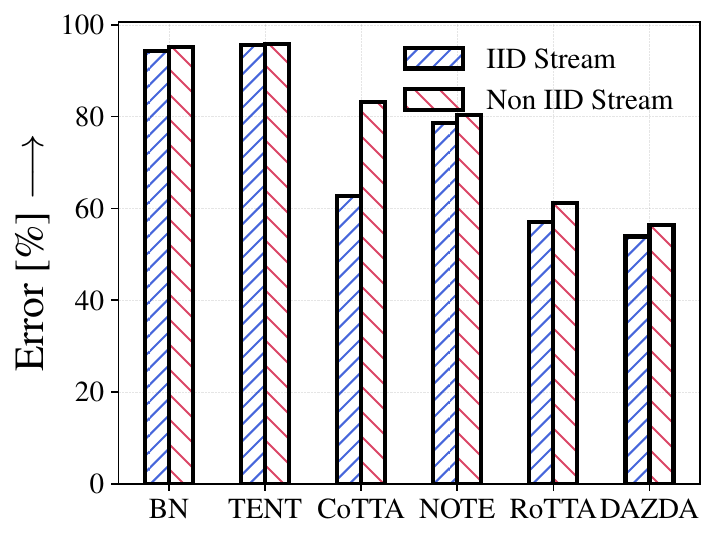}
                \caption{Tiny Imagenet }
                \label{fig: tiny_imagenet_performance}
        \end{subfigure}
         \setlength\abovecaptionskip{-0.05cm}
        \caption{Performance Comparison in Different Popular Corruption Benchmark Datasets (Unseen Corruptions). \vspace{-0.4cm} }\label{fig:performance}
\end{figure*}

\subsection{Corruption-Aware Batch Normalization}\label{sec:batch-normalization}
Due to sudden corruption, the normalization statistics $(\bar{\mu}, \bar{\sigma}^2)$ estimated on uncorrupted training data become unreliable. Although using a specific sub-network with normalization statistics $(\mu_s, \sigma_s^2)$ would be feasible, we can use  the samples stored in the memory bank to further refine our estimation of the ongoing statistics $(\mu_t, \sigma_t^2)$. To this end, we use \gls{bn} \cite{ioffe2015batch}.  Let  $A^{l}\,\epsilon \,\mathbb{R}^{B\times Ch^l\times N^l}$ be a batch of activation tensors of the $l^{th}$ convolutional layer, where $B$ corresponds to the batch size, $Ch^l$ denotes the number of channels in $l^{th}$ layer and $N_{l}$ is the dimension of activations in each channel. A \gls{bn} layer first calculates $\mu_{Ch} =\frac{1}{\left | B \right |\left | N^l \right |}\sum _{b\epsilon B, \:n\epsilon N^l}(A^l)$ and $\sigma_{Ch} =\frac{1}{\left | B \right |}\sum _{b\epsilon B }(A^l \:-\mu_{Ch})^{2}$ and subtracts $\mu_{Ch}$ from all input activations in the channel. Subsequently, \gls{bn} divides the centered activation by the standard deviation $\sigma_{Ch}$. Normalization is applied:
\vspace{-.2cm}
\begin{equation}
BN\left ( A^l_{b,Ch,N^l} \right ) \leftarrow \gamma \times \frac{A^l_{b,Ch,N\,^l}-\mu_{Ch}}{\sqrt{\sigma_{Ch}^2 + \epsilon}} \:+ \beta \;\;\;\;\forall \; b,Ch,N^l
\end{equation}
Here, $\gamma$ and $\beta$ are the affine scaling and shifting parameters followed by normalization, while $(\epsilon > 0)$ is a small constant added for numerical stability. The normalized and affine transformed outputs are passed to the next $(l+1)^{th}$ layer, while the normalized output is kept to the $l^{th}$ layer. \gls{bn} also keeps track of the estimate of running mean and variance to use during the inference phase as a global estimate of normalization statistics, and $\gamma$ and $\beta$ are optimized with the other \gls{dnn} parameters through back propagation. 

Whenever the corruption changes, the projection from the corruption encoder matches with the closest corruption centroid in the latent space. At the same time, samples affected by the new corruption are being stored in the memory bank. As we constrict the memory bank to have certain amount of samples from a particular class, due to our design of non \gls{iid} real world data stream, initially there will be samples from previous corruption distribution also on the memory bank. Fig. \cref{fig:corruption tsne} shows that the projections for even the unknown corruption get clustered nearby in the latent space. When the samples of the memory bank become representative of the current corruption, they should have low variance among their cosine similarity with current closest corruption centroid $C_{cur}$. Therefore, when the change in corruption is detected and the variance of cosine similarity from the centroid becomes lower than  $\varphi_{thresh}$, current \gls{dnn} normalization statistics are updated as:
\begin{equation} \label{eq: ca BatchNorm}
    \begin{split}
    \mu_s = (1-m)\cdot \mu_s \,+\, m \cdot \hat{\mu}^t\\
    \sigma_s^2 = (1-m)\cdot \sigma_s^2 \,+\, m\cdot \hat{\sigma}_t^2
    \end{split}
\end{equation}
where $m$ is the momentum and $(\hat{\mu_t}, \hat{\sigma}_t^2)$ are the current  normalization statistics of different layers of the \gls{dnn}, which we obtain by making one forward pass using the samples in the memory bank.

\subsection{Corruption-Aware Real-Time \gls{dnn} Adaptation}\label{sec:real-time-adaptation}

Adapting the parameters of the current \gls{dnn} in an unsupervised manner usually needs careful selection of hyper-parameters. To avoid this issue, only the sub-network is adapted. As explained earlier, we make one forward pass with the samples in the memory bank and the fixed Gaussian noise to calculate the normalization statistics of the current ongoing corruption and current sub-network fingerprint. The mean of the corruption embedding $\bar{C} = \frac{1}{\mathcal{N}}\sum_{i=1}^{\mathcal{N}} C^i$ of the samples in the memory bank and the sub-network projection $S$ is also calculated. Specifically, the tuneable parameters of the sub-network (shift $\beta$ and scale $\gamma$ parameters of the \gls{bn} layer and final fully connected layers) are updated using gradient descent to minimize the following unsupervised loss:
\begin{equation} \label{eqn:loss_fn}
    \mathcal{L}_u = exp(- \, S \, \cdot \, \bar{C})
\end{equation} 
Notice that we do not assume any access to labeled data, and do not use pseudo-labels (i.e., \gls{dnn} prediction) as labels. Conversely, we minimize the loss between corruption embedding and model state in the latent space which is not disrupted by the distribution of labels in the batch.

\section{Experimental Results} \label{sec:results}

\noindent \textbf{Datasets.} We use the tool described in \cite{michaelis2019benchmarking} to synthetically generate realistic corruptions for CIFAR-10 and CIFAR-100 datasets. Both CIFAR-10 and CIFAR-100 have 50,000 training images and 10,000 testing images. In line with prior work, we use the following 15 different corruptions of different categories which are ``Noise'' (\textit{Gaussian, shot, impulse}), ``Blur'' (\textit{defocus, glass, motion, zoom}), ``Weather'' (\textit{snow, frost, fog, bright}) and ``Digital'' (\textit{contrast, elastic, pixelate, JPEG}) to train our corruption encoder. For fair comparison, we evaluate the performance on 4 corruptions (\textit{Gaussian blur, saturate, spatter, speckle noise}) which are unseen during training phase. In line with \cite{wang2020tent,gong2022note, wang2022continual, zhang2022practical}, we consider the corruptions in their highest severity. \smallskip

%\subsection{Implementation Details}
\vspace{-.2cm}
\subsection{Comparison with State-of-the-Art Benchmarks} \label{sec:baseline}
\cref{fig: cifar10_performance} and \cref{fig: cifar100_performance} show the performance \FW as compared to other state of the art approaches. 
% BN (ICML 2020) \cite{nado2020evaluating}, TENT (ICLR 2021) \cite{wang2020tent}, CoTTA (CVPR 2022) \cite{wang2022continual}, NOTE (NeurIPS 2022) \cite{gong2022note}, RoTTA (CVPR 2023) \cite{yuan2023robust}. 
We show results with ideal \gls{iid} assumption and a more realistic non-\gls{iid} assumption (i.e., samples are correlated). 

As we can observe, \FW performs consistently better in both datasets and across both setups, with \FW improving the performance by 10.4\% on CIFAR-10 test corruptions and 5.7\% on CIFAR-100 compared to the 2nd best performing baseline RoTTA. Among the other considered baselines, TENT, CoTTA and BN achieves poor performance for non-\gls{iid} samples. Since \FW starts from a bootstrapped sub-network in changed corruption domain and update only with samples from our memory bank and corruption latent space, \FW performs consistently well in both cases. Since NOTE  is equipped to handle correlation among online data batches, there is no significant performance drop for non-\gls{iid} assumption. However, NOTE resets the \gls{dnn} after evaluation on each corruption type which is unrealistic as it does not have the capability to know when current corruption domain is changing. Thus, due to error accumulation for continuous adaptation, the performance is fairly poor. 
\FW does not accumulate errors from the previous corruption domain, as in every new corruption domain we start from a new \gls{dnn} state.  

\vspace{-.1 cm}
\subsection{Sensitivity to Correlated Samples}
\vspace{-.1 cm}
The performance of \FW does not depend on the label space  to bootstrap from an appropriate sub-network. To prove this point, we plot the average performance in the first five batches of incoming data by varying the value of the correlation parameter $\delta$ for the CIFAR-100 dataset using sequences similar to those used for \cref{fig: cifar100_performance}. 
From \cref{fig: first n_batch performance} it emerges that the correlation does not have a sensible effect on \FW. Furthermore, as the value of  $\delta$ decreases, the performance of BN and CoTTA drastically decreases, and for the best method RoTTA, performance starts to plummet when encountering severe correlation. Indeed,  with high correlation among samples, there is not enough sample diversity to have a stable update of the \gls{dnn}. However, \FW can reliably sense corruption drift even with a single sample and bootstrap with the most similar sub-network stored in the buffer. Thus, even for extremely correlated samples, we have a well-performing sub-network from the early instances of data batches after incurring in corruption. This indicates the efficacy and reliability of \FW in critical mobile edge computing scenarios.

\begin{figure}[!h]
    \centering
    \includegraphics[width=0.45\textwidth]{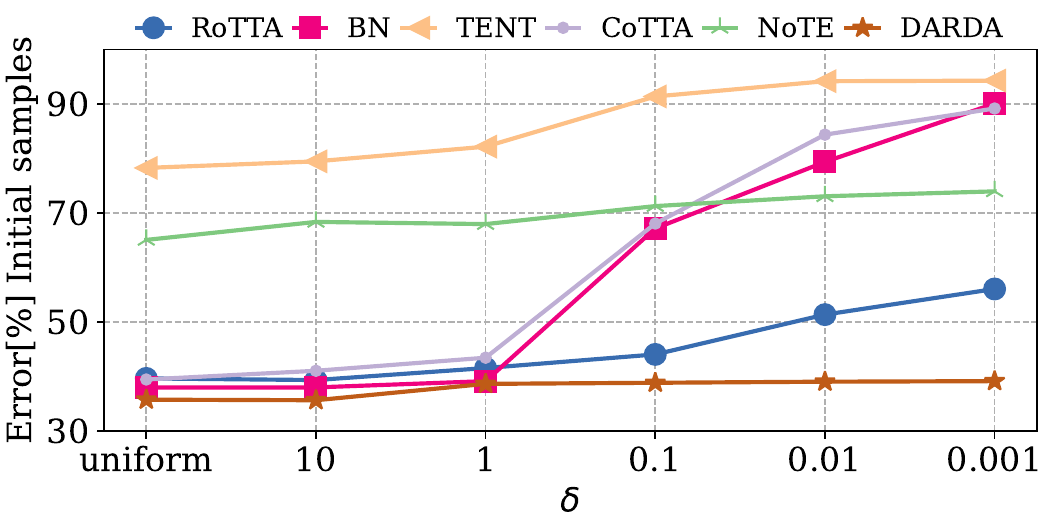} 
    \setlength\abovecaptionskip{-0.05cm}
    \caption{Performance across first five data batches for continuously incurred corruption on  CIFAR-100. \vspace{-0.3cm}}
    \label{fig: first n_batch performance}
\end{figure}

% Other than preventing error accumulation when shifting 
% \FW can reliably sense corruption drift even with a single sample and bootstrap with the most similar sub network.

\subsection{Effect of Batch Size and Dirichlet Parameter}
\vspace{-.1cm}
We can observe from \cref{fig:gamma and batch} that the batch size and Dirichlet parameter do not have a significant effect on performance of \FW. This is because the corruption signature can be extracted even with a single sample. Moreover, by updating the normalization statistics of the \gls{bn} layer using a memory bank that enforces diversity among samples, we ensure that samples are representative of the ongoing corruption. Among other approaches, a higher value of batch size leads to high performance gain for TENT, BN and CoTTA. Especially for CoTTA, with batch size greater than 128, the performance reaches up to RoTTA. Also, the overall effect of Dirichlet parameter $\delta$ is less prominent than the initial batches.

\vspace{-.2cm}

\begin{figure}[h]
    \centering
    \includegraphics[width=0.5\textwidth]{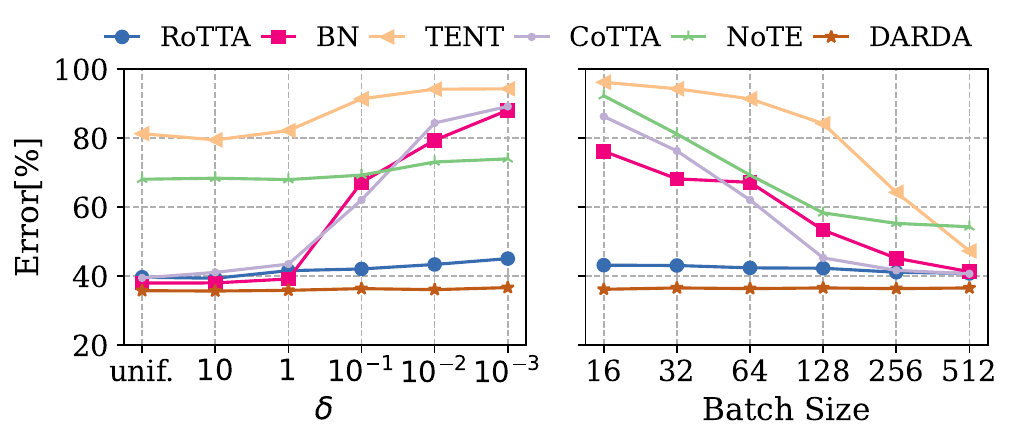} 
    \setlength\abovecaptionskip{-0.5cm}
    \caption{Performance vs correlation coefficient and batch size.\vspace{-.4cm}}
    
    \label{fig:gamma and batch}
\end{figure}
%\vspace{-.4cm}

\begin{figure} [!h]
        \centering
        \begin{subfigure}[b]{0.23\textwidth}
                \centering
                \includegraphics[width=\linewidth]{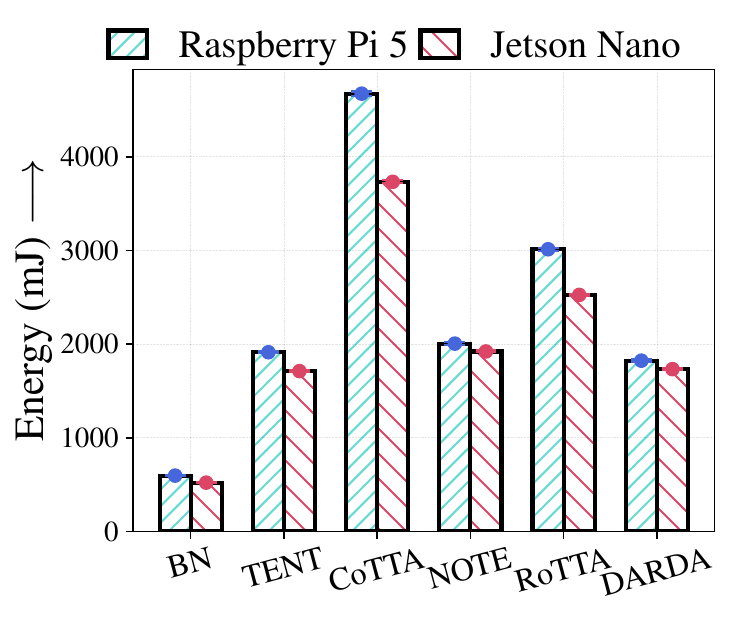}
                \caption{CIFAR-100 }
                \label{fig: energy-cifar}
        \end{subfigure}
        \begin{subfigure}[b]{.23\textwidth}
                \centering
                \includegraphics[width=\linewidth]{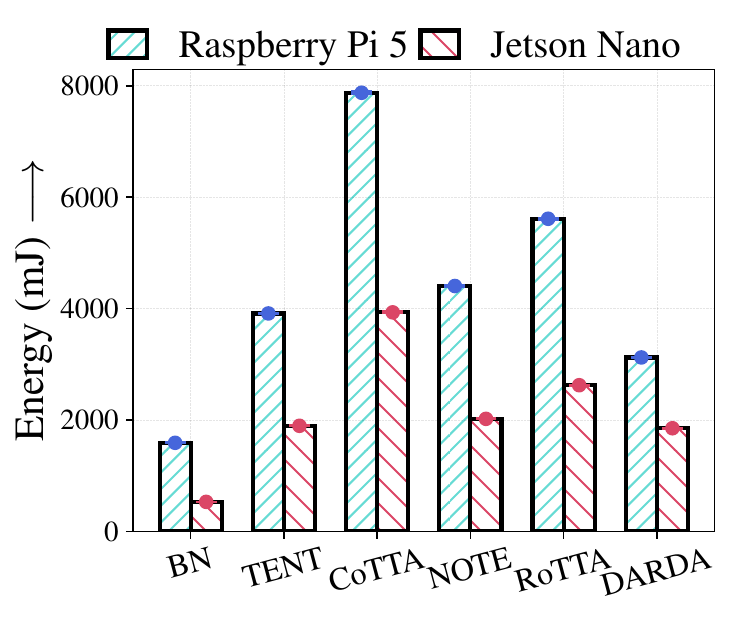}
                \caption{Tiny Imagenet}
                \label{fig: energy-tiny-imagenet}
        \end{subfigure}
         \setlength\abovecaptionskip{-0.05cm}
        \caption{Energy consumption on common edge devices for a batch of 64 data samples. \vspace{-0.3cm}}\label{fig: energy_performance}

\end{figure}

\subsection{Evaluation of Catastrophic Forgetting}
\vspace{-.1cm}
The adaptation of \gls{dnn} should not result in performance degradation on uncorrupted inputs, since in most real-life scenarios uncorrupted data is most common. However, during \gls{tta}, the \gls{dnn} might get specialized in certain corruptions and fail to deliver performance. \cref{fig:forgetting} shows the performance showing 5 uncorrupted data batches from CIFAR-10 after two sequences,  which are (on the left subfigure) saturate $\rightarrow$ Gaussian blur $\rightarrow$ spatter $\rightarrow$ speckle and (on the right subfigure) Gaussian blur $\rightarrow$ saturate $\rightarrow$ spatter $\rightarrow$ speckle. \cref{fig:forgetting} concludes that conversely from \FW, state-of-the-art approaches incur in catastrophic forgetting. The best-performing baseline RoTTA has its accuracy degraded by up to 22\% when uncorrupted data is fed after speckle noise. The classification error is particularly high when uncorrupted data is fed after a more severe corruption type (e.g, Gaussian blur) than the milder one (e.g, spatter).

\begin{figure}[h]
    \centering
    \includegraphics[width=0.5\textwidth]{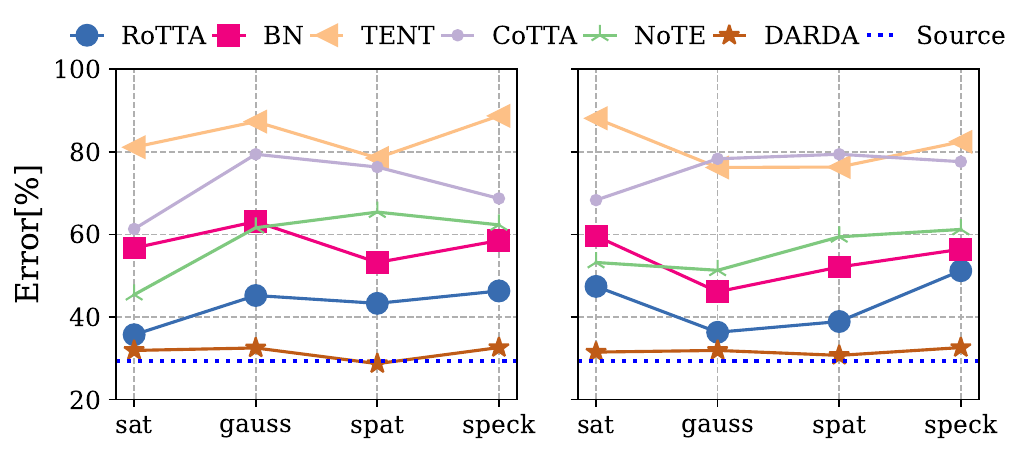} 
    
    \setlength\abovecaptionskip{-0.1cm}
    \caption{Performance when 5 uncorrupted data batches are fed after every adaptation,  for CIFAR-100 and $\delta=0.01$. \vspace{-0.4cm}}
    \label{fig:forgetting}
\end{figure}

\subsection{On-Device Dynamic Adaptation Efficiency}

In this section, we evaluate the on-device performance of \FW in diverse edge platforms. We select commonly available Raspberry-Pi-5 and Nvidia Jetson-Nano, since they are representative of resource-constrained devices widely applied for mobile vision applications. 

% The Raspberry Pi runs a quad-core ARM A76 SoC running at upto 2.4 GHz with 8 GB LPDDR4 memory. Apart from being powered by quad-core ARM Cortex A57 CPU, Jetson-Nano also has hardware acceleration capabilities as it is equipped with 128 core Maxwell GPU. 

% \begin{figure}[h]
% 	\centering
% 	\includegraphics[width=\columnwidth]{fig/power cicuit.pdf} 
% 	\caption{Power Measurement Setup.}
% 	\label{fig: power measurement}

% \end{figure}
% \textbf{Computational Efficiency.}~
\FW adapts using one backward pass to update the sub-network only when a corruption is perceived by the corruption extractor. The corruption extractor and corruption encoder are also active for each data sample. \cref{tab:forward&backward} reports the average number of \gls{mac} operations for the forward pass of each method. For the backward pass, the average number of samples for different corruptions with which the backward pass was called is reported. To calculate \FW forward pass \gls{mac}, all the operations involved in the corruption extractor, corruption encoder and sub-network encoder are summed with the operation performed by different sub-networks. To support continuous adaptation without error accumulation,  CoTTA and RoTTA continuously perform two forward passes with the original data sample and another augmented sample respectively. Thus, their number of operations is two times more than BN, TENT and CoTTA. Although BN, TENT and NOTE have less forward computation than \FW, the unrealistic assumption of \gls{iid} data stream and episodic adaptation  -- notice that the \gls{dnn} state is continuously reset after adaptation -- which makes them not applicable in real-time mobile edge applications. From \cref{fig:latency_performance} it is also evident that \FW incurs lower CPU and GPU latency compared to closest performing benchmarks.

% \textbf{Efficiency of Cache Usage.}~
% A mobile edge device which might request diverse task request during its deployment, its very important that our dynamic adaptation task should consume as low memory as possible while also provide stable performance across corruption distribution shift and changing scenario. Since, there is constant cache memory usage for a particular model and optimizer which scales linearly with parameter increase, we work with the garbage collector module to log the tensor size of intermediate varaibles on average along the adaptation process. 
From \cref{tab:forward&backward}, we can be observe that there is an excessive amount of cache usage during adaptation, except for BN, which can slow down or even block some other tasks we are interested in using the same device. For example, the closest performing baseline RoTTA in continual adaptation settings needs $8.78x$ cache than \FW. Although BN needs less cache usage than \FW to operate as it does not need to store gradient for backward pass, it performs poorly,  as shown in \cref{sec:baseline}.
From \cref{fig: energy_performance}, it is observed that \FW uses $2.9\times$ less energy than the closest performing benchmark in terms of performance. To calculate the energy consumption value, we use the setup in \cref{fig: power measurement}.

% \begin{figure}[!h]
%     \centering
%     \includegraphics[width=\linewidth]{fig/power.pdf} 
%     %\setlength\abovecaptionskip{-0.4cm}
%     \caption{Power Consumption on Edge Devices for CIFAR100 dataset }
%     \label{fig:cifar10_power}
% \end{figure}

\begin{figure} [!h]
        \centering
        \begin{subfigure}[b]{0.23\textwidth}
                \centering
                \includegraphics[width=\linewidth]{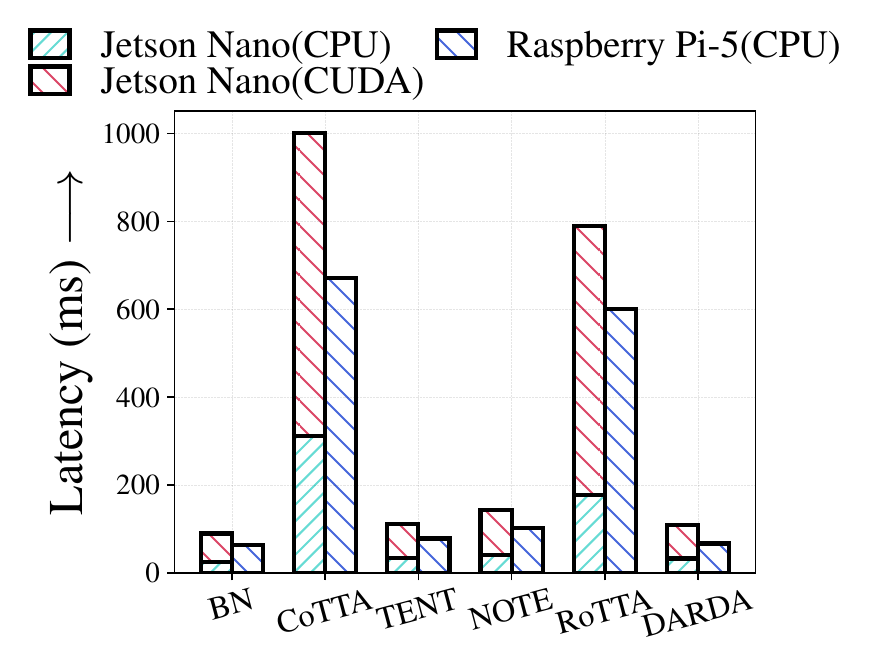}
                \caption{CIFAR-100 }
                \label{fig: latency-cifar}
        \end{subfigure}
        \begin{subfigure}[b]{.23\textwidth}
                \centering
                \includegraphics[width=\linewidth]{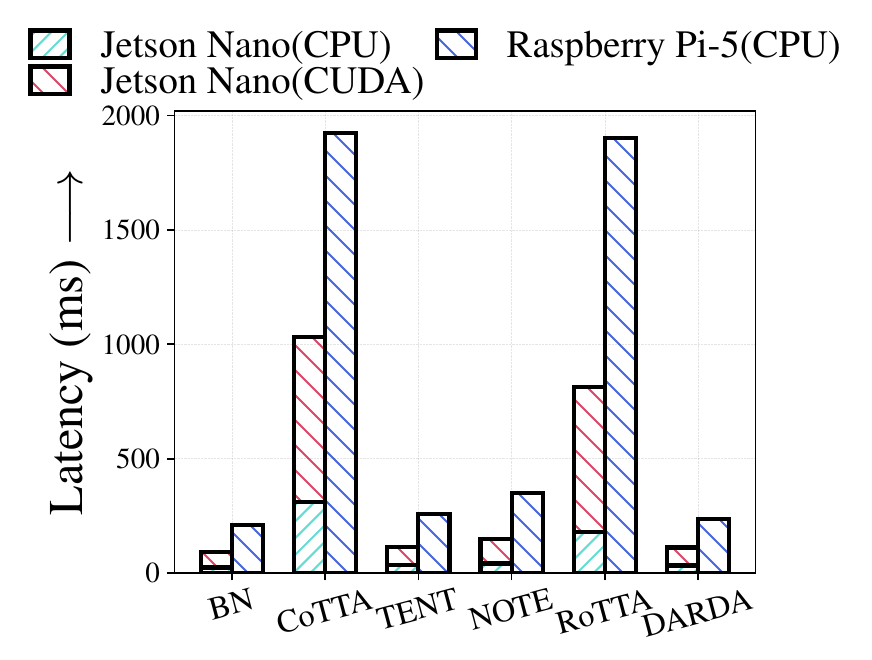}
                \caption{Tiny Imagenet}
                \label{fig: latency-tiny-imagenet}
        \end{subfigure}
         \setlength\abovecaptionskip{-0.05cm}
        \caption{Adaptation Latency on common edge devices for a batch of 64 data samples.\vspace{-0.3cm} }\label{fig:latency_performance}

\end{figure}

\renewcommand{\arraystretch}{0.65}
\begin{table}[t]
\centering
\resizebox{\columnwidth}{!}{\begin{tabular}{cccc}
\toprule
\textbf{Method} & \textbf{Cache(Mb)}  & \begin{tabular}[c]{@{}c@{}}\textbf{Average no. samples} \\ \textbf{Forward MACS}\end{tabular} & \begin{tabular}[c]{@{}c@{}}\textbf{Average no. Samples} \\  \textbf{Backward Pass}\end{tabular} \\ \midrule
BN     & 67       & $1.28 \times 10^{12}$                                                     & 0                                                                           \\ \midrule
TENT   & 914       & $1.28 \times 10^{12}$                                                     & 10,000                                                                      \\ \midrule
CoTTA  & 4271       & $2.56 \times 10^{12}$                                                     & 10,000                                                                      \\ \midrule
NOTE   & 1003     & $1.28 \times 10^{12}$                                                     & 10,000                                                                      \\ \midrule
RoTTA  & 2735      & $2.56 \times 10^{12}$                                                     & 10,000                                                                      \\ \midrule
DARDA  & 312      & $1.82 \times 10^{12}$                                                     & 2294                                                                        \\ \bottomrule

\end{tabular}} 
\caption{Comparison of Computation and Average Cache Usage of \FW  and other approaches while performing continuous adaptation on CIFAR-100. \vspace{-0.5cm}}
\label{tab:forward&backward}
\end{table}

\subsection{Impact of Submodules of \FW} 

To investigate the contribution of different components toward performance gain,  we replace different parts of \FW using different alternative options. We report the related performance in \cref{tab:effect of individual component of darda}. We consider (i) \FW without Corruption Extractor, where  corrupted data is directly projected into latent space; (ii) \FW without context-aware \gls{bn} \& Adaptation, directly using the \gls{dnn} constructed from current sub-network signature for prediction; (iii) \FW with context-aware \gls{bn} \& without fine tuning, we update the normalization values of the \gls{bn} layers using samples from the memory bank but do not update the tunable parameters; (iv) \FW with context-aware \gls{bn} and Entropy Minimization, where we update the tunable parameters by minimizing entropy of predictions. 

From \cref{tab:effect of individual component of darda} it can be observed that the corruption extractor is   crucial for the performance as erroneous corruption projection would select the wrong sub-network. Creating corruption projections using only corrupted data leads to performance drop of 20.4\% and 16.6\% on CIFAR-10 and CIFAR-100 unseen corruption respectively. Context-aware \gls{bn} with adaptation is also important as the accuracy reduces by 6.4\% and 5.1\% respectively for test corruptions on CIFAR-10 and CIFAR-100, respectively. It is an interesting observation that for CIFAR-10, updating the tunable parameters by minimizing entropy of predictions degrades the performance by 0.1\% rather that improving. However, our loss $\mathcal{L}_n$ in \cref{eqn:loss_fn} results up  to 1.2\% performance gain which proves that a learned cross-modal latent space can guide \gls{dnn} adaptation. 

\renewcommand{\arraystretch}{0.9}
\begin{table}[t]
\centering
\normalsize
\begin{tabular}{ccc}
\toprule
\begin{tabular}[c]{@{}c@{}} \textbf{Variants of \FW}\end{tabular}                              & \begin{tabular}[c]{@{}c@{}}\textbf{Error {[}\%{]}} \\\textbf{(CIFAR-10)}\end{tabular} & \begin{tabular}[c]{@{}c@{}} \textbf{Error {[}\%{]}} \\\textbf{(CIFAR-100)}\end{tabular} \\ \midrule
\begin{tabular}[c]{@{}c@{}} w/o Corruption Extractor\end{tabular}                              & 36.7                                                              & 56.9                                                               \\ \midrule
\begin{tabular}[c]{@{}c@{}} w/o Context Aware BN\\ \& Adaptation\end{tabular}                & 22.7                                                              & 41.5                                                               \\ \midrule
\begin{tabular}[c]{@{}c@{}} with Context Aware BN\\  w/o Adaptation\end{tabular}      & 23.8                                                              & 42.9                                                               \\ \midrule
\begin{tabular}[c]{@{}c@{}} with Context Aware BN\\ \& Entropy Minimization\end{tabular} & 23.9                                                              & 42.7                                                               \\ \midrule
Ours                                                                                               & \textbf{16.3}                                                              & \textbf{36.4}                                                               \\ \bottomrule
\end{tabular}
\caption{Effect of Individual Components of \FW.\vspace{-0.5cm}}
\label{tab:effect of individual component of darda}
%\vspace{-1.2 cm}
\end{table}

\section{Conclusion} \label{conclusion}

% Ensuring that sudden and unexpected corruptions (e.g., snowy or foggy conditions) do not compromise the accuracy of \glspl{dnn} is of paramount importance in mobile edge computing. Existing work assumes the availability of corrupted input data for all output classes, which is hardly realistic since inputs are usually classified in a time-sequential fashion. 
\vspace{-.2cm}
In this work, we have proposed  Domain-Aware Real-TimeDynamic Neural Network Adaptation (\FW). \FW adapts the \gls{dnn} to \textit{previously unseen} corruptions in an \emph{unsupervised fashion} by (i) estimating the latent representation of the ongoing corruption; (ii) selecting the sub-network whose associated corruption is the closest in the latent space to the ongoing corruption; and (iii) adapting \gls{dnn} state, so that its representation matches the ongoing corruption. This way, \FW is more resource-efficient and can swiftly adapt to new distributions without requiring a large variety of input data. Through experiments with two popular mobile edge devices -- Raspberry Pi and NVIDIA Jetson Nano --  we show that \FW reduces energy consumption and average cache memory footprint respectively by $1.74\times$ and $2.64\times$ with respect to the state of the art, while increasing the performance by $10.4\%$, $5.7\%$ and $4.4\%$ on CIFAR-10, CIFAR-100 and TinyImagenet. 

\section*{Acknowledgment of Support and Disclaimer}
This work has been funded in part by the National Science Foundation under grants CNS-2134973, ECCS-2229472, CNS-2312875 and ECCS-2329013, by the Air Force Office of Scientific Research under contract number FA9550-23-1-0261, by the Office of Naval Research under award number N00014-23-1-2221, and by the Air Force Research Laboratory via \textit{Open Technology and Agility for Innovation (OTAFI)} under transaction number FA8750-21-9-9000 between SOSSEC, Inc. and the U.S. Government. The U.S. Government is authorized to reproduce and distribute reprints for Governmental purposes notwithstanding any copyright notation thereon. The views and conclusions contained herein are those of the authors and should not be interpreted as necessarily representing the official policies or endorsements, either expressed or implied, of U.S. Air Force, U.S. Navy or the U.S. Government. \vspace{-0.3cm}

%%%%%%%%% REFERENCES
{\small
\bibliographystyle{ieee_fullname}
\bibliography{bibliography,bib-francesco}

\begin{thebibliography}{10}\itemsep=-1pt

\bibitem{shawabka2020exposing}
Amani Al-Shawabka, Francesco Restuccia, Salvatore D'Oro, Tong Jian, Bruno~Costa Rendon, Nasim Soltani, Jennifer Dy, Kaushik Chowdhury, Stratis Ioannidis, and Tommaso Melodia.
\newblock {Exposing the Fingerprint: Dissecting the Impact of the Wireless Channel on Radio Fingerprinting}.
\newblock {\em Proc. of IEEE Conference on Computer Communications (INFOCOM)}, 2020.

\bibitem{gao2022visual}
Yunhe Gao, Xingjian Shi, Yi Zhu, Hao Wang, Zhiqiang Tang, Xiong Zhou, Mu Li, and Dimitris~N Metaxas.
\newblock {Visual Prompt Tuning for Test-Time Domain Adaptation}.
\newblock {\em arXiv preprint arXiv:2210.04831}, 2022.

\bibitem{gong2022note}
Taesik Gong, Jongheon Jeong, Taewon Kim, Yewon Kim, Jinwoo Shin, and Sung-Ju Lee.
\newblock Note: Robust continual test-time adaptation against temporal correlation.
\newblock {\em Advances in Neural Information Processing Systems}, 35:27253--27266, 2022.

\bibitem{goyal2022test}
Sachin Goyal, Mingjie Sun, Aditi Raghunathan, and J~Zico Kolter.
\newblock Test time adaptation via conjugate pseudo-labels.
\newblock {\em Advances in Neural Information Processing Systems}, 35:6204--6218, 2022.

\bibitem{han2021legodnn}
Rui Han, Qinglong Zhang, Chi~Harold Liu, Guoren Wang, Jian Tang, and Lydia~Y Chen.
\newblock Legodnn: block-grained scaling of deep neural networks for mobile vision.
\newblock In {\em Proceedings of the 27th Annual International Conference on Mobile Computing and Networking}, pages 406--419, 2021.

\bibitem{hayes2020remind}
Tyler~L Hayes, Kushal Kafle, Robik Shrestha, Manoj Acharya, and Christopher Kanan.
\newblock {Remind Your Neural Network to Prevent Catastrophic Forgetting}.
\newblock In {\em Proceedings of European Conference on Computer Vision (ECCV)}, pages 466--483. Springer, 2020.

\bibitem{hendrycks2019benchmarking}
Dan Hendrycks and Thomas Dietterich.
\newblock Benchmarking neural network robustness to common corruptions and perturbations.
\newblock {\em arXiv preprint arXiv:1903.12261}, 2019.

\bibitem{huang2021neighbor2neighbor}
Tao Huang, Songjiang Li, Xu Jia, Huchuan Lu, and Jianzhuang Liu.
\newblock Neighbor2neighbor: Self-supervised denoising from single noisy images.
\newblock In {\em Proceedings of the IEEE/CVF conference on computer vision and pattern recognition}, pages 14781--14790, 2021.

\bibitem{ioffe2015batch}
Sergey Ioffe and Christian Szegedy.
\newblock Batch normalization: Accelerating deep network training by reducing internal covariate shift.
\newblock In {\em International conference on machine learning}, pages 448--456. pmlr, 2015.

\bibitem{kurmi2021domain}
Vinod~K Kurmi, Venkatesh~K Subramanian, and Vinay~P Namboodiri.
\newblock Domain impression: A source data free domain adaptation method.
\newblock In {\em Proceedings of the IEEE/CVF winter conference on applications of computer vision}, pages 615--625, 2021.

\bibitem{li2020model}
Rui Li, Qianfen Jiao, Wenming Cao, Hau-San Wong, and Si Wu.
\newblock Model adaptation: Unsupervised domain adaptation without source data.
\newblock In {\em Proceedings of the IEEE/CVF conference on computer vision and pattern recognition}, pages 9641--9650, 2020.

\bibitem{lin2023multimodality}
Zhiqiu Lin, Samuel Yu, Zhiyi Kuang, Deepak Pathak, and Deva Ramanan.
\newblock {Multimodality Helps Unimodality: Cross-Modal Few-Shot Learning with Multimodal Models}.
\newblock In {\em Proceedings of the IEEE/CVF Conference on Computer Vision and Pattern Recognition (CVPR)}, pages 19325--19337, 2023.

\bibitem{mansour2023zero}
Youssef Mansour and Reinhard Heckel.
\newblock Zero-shot noise2noise: Efficient image denoising without any data.
\newblock In {\em Proceedings of the IEEE/CVF Conference on Computer Vision and Pattern Recognition}, pages 14018--14027, 2023.

\bibitem{michaelis2019benchmarking}
Claudio Michaelis, Benjamin Mitzkus, Robert Geirhos, Evgenia Rusak, Oliver Bringmann, Alexander~S Ecker, Matthias Bethge, and Wieland Brendel.
\newblock Benchmarking robustness in object detection: Autonomous driving when winter is coming.
\newblock {\em arXiv preprint arXiv:1907.07484}, 2019.

\bibitem{nado2020evaluating}
Zachary Nado, Shreyas Padhy, D Sculley, Alexander D'Amour, Balaji Lakshminarayanan, and Jasper Snoek.
\newblock Evaluating prediction-time batch normalization for robustness under covariate shift.
\newblock {\em ICML 2020 Workshop on Uncertainty and Robustness in Deep Learning}, 2020.

\bibitem{nguyen2023tipi}
A~Tuan Nguyen, Thanh Nguyen-Tang, Ser-Nam Lim, and Philip~HS Torr.
\newblock Tipi: Test time adaptation with transformation invariance.
\newblock In {\em Proceedings of the IEEE/CVF Conference on Computer Vision and Pattern Recognition}, pages 24162--24171, 2023.

\bibitem{niu2022efficient}
Shuaicheng Niu, Jiaxiang Wu, Yifan Zhang, Yaofo Chen, Shijian Zheng, Peilin Zhao, and Mingkui Tan.
\newblock Efficient test-time model adaptation without forgetting.
\newblock In {\em International conference on machine learning}, pages 16888--16905. PMLR, 2022.

\bibitem{radford2021learning}
Alec Radford, Jong~Wook Kim, Chris Hallacy, Aditya Ramesh, Gabriel Goh, Sandhini Agarwal, Girish Sastry, Amanda Askell, Pamela Mishkin, Jack Clark, et~al.
\newblock Learning transferable visual models from natural language supervision.
\newblock In {\em International conference on machine learning}, pages 8748--8763. PMLR, 2021.

\bibitem{sakaridis2021acdc}
Christos Sakaridis, Dengxin Dai, and Luc Van~Gool.
\newblock Acdc: The adverse conditions dataset with correspondences for semantic driving scene understanding.
\newblock In {\em Proceedings of the IEEE/CVF International Conference on Computer Vision}, pages 10765--10775, 2021.

\bibitem{sayed2012learning}
Moamar Sayed-Mouchaweh and Edwin Lughofer.
\newblock {\em {Learning in Non-Stationary Environments: Methods and Applications}}.
\newblock Springer Science \& Business Media, 2012.

\bibitem{schneider2020improving}
Steffen Schneider, Evgenia Rusak, Luisa Eck, Oliver Bringmann, Wieland Brendel, and Matthias Bethge.
\newblock Improving robustness against common corruptions by covariate shift adaptation.
\newblock {\em Advances in neural information processing systems}, 33:11539--11551, 2020.

\bibitem{tian2020predicted}
Feng Tian, Yue Yu, Xu Yuan, Bin Lyu, and Guan Gui.
\newblock {Predicted Decoupling for Coexistence Between WiFi and LTE in Unlicensed Band}.
\newblock {\em IEEE Transactions on Vehicular Technology}, 69(4):4130--4141, 2020.

\bibitem{wang2020tent}
Dequan Wang, Evan Shelhamer, Shaoteng Liu, Bruno Olshausen, and Trevor Darrell.
\newblock Tent: Fully test-time adaptation by entropy minimization.
\newblock {\em arXiv preprint arXiv:2006.10726}, 2020.

\bibitem{wang2021tent}
Dequan Wang, Evan Shelhamer, Shaoteng Liu, Bruno Olshausen, and Trevor Darrell.
\newblock Tent: Fully test-time adaptation by entropy minimization.
\newblock In {\em International Conference on Learning Representations}, 2021.

\bibitem{wang2022continual}
Qin Wang, Olga Fink, Luc Van~Gool, and Dengxin Dai.
\newblock Continual test-time domain adaptation.
\newblock In {\em Proceedings of the IEEE/CVF Conference on Computer Vision and Pattern Recognition}, pages 7201--7211, 2022.

\bibitem{wen2023adaptivenet}
Hao Wen, Yuanchun Li, Zunshuai Zhang, Shiqi Jiang, Xiaozhou Ye, Ye Ouyang, Yaqin Zhang, and Yunxin Liu.
\newblock Adaptivenet: Post-deployment neural architecture adaptation for diverse edge environments.
\newblock In {\em Proceedings of the 29th Annual International Conference on Mobile Computing and Networking}, pages 1--17, 2023.

\bibitem{wilson2020survey}
Garrett Wilson and Diane~J Cook.
\newblock A survey of unsupervised deep domain adaptation.
\newblock {\em ACM Transactions on Intelligent Systems and Technology (TIST)}, 11(5):1--46, 2020.

\bibitem{yuan2023robust}
Longhui Yuan, Binhui Xie, and Shuang Li.
\newblock Robust test-time adaptation in dynamic scenarios.
\newblock In {\em Proceedings of the IEEE/CVF Conference on Computer Vision and Pattern Recognition}, pages 15922--15932, 2023.

\bibitem{zhang2022practical}
Daqing Zhang, Dan Wu, Kai Niu, Xuanzhi Wang, Fusang Zhang, Jian Yao, Dajie Jiang, and Fei Qin.
\newblock {Practical Issues and Challenges in CSI-based Integrated Sensing and Communication}.
\newblock {\em arXiv preprint arXiv:2204.03535}, 2022.

\bibitem{zhu2022multimodal}
Tong Zhu, Leida Li, Jufeng Yang, Sicheng Zhao, Hantao Liu, and Jiansheng Qian.
\newblock Multimodal sentiment analysis with image-text interaction network.
\newblock {\em IEEE Transactions on Multimedia}, 2022.

\end{thebibliography}
}

\clearpage

    \setcounter{section}{0}
    \renewcommand{\thesection}{S\arabic{section}}
    \setcounter{equation}{0}
    \renewcommand{\theequation}{S\arabic{equation}}
    \newcounter{offset}
    \setcounter{offset}{\value{figure}}
    \renewcommand{\thefigure}{S\the\numexpr\value{figure}-\value{offset}\relax}
\section{Different Domain Generalization Setup}
The problem of adapting a \gls{dnn} to tackle real life data corruption at the edge can be formulated by different kind of \gls{dg} settings based on the nature of data stream and learning paradigm. \cref{fig:tta} summarizes the four main \gls{dg} approaches in literature, namely \gls{ft}, \gls{uda}, \gls{sfda} and \gls{tta}. \smallskip

\noindent \textbf{Fine Tuning (FT)} adapts a \gls{dnn} by making it match labeled test data  \cite{lin2023multimodality,gao2022visual}. FT approaches includes Few-Shots Learning (FSL), among others \cite{tian2020predicted}. The downside of FT is that it requires labeled test data and is performed offline, thus they are hardly applicable in a mobile edge computing context. Moreover, it does not take into account samples from the previous domain, so it  incurs in catastrophic forgetting \cite{hayes2020remind}. \smallskip

\noindent \textbf{Unsupervised Domain Adaptation (UDA).}~This approach addresses the issues of FT considering samples from the previous domain and thereby eliminating the need of labels from the new domain \cite{wilson2020survey}. However, as FT, UDA assumes that we can simultaneously access (unlabeled) samples from the current domain and from the prior domain, which is not always the case.  In stark opposition, our goal is achieve real-time adaptation of a \gls{dnn} in dynamic and uncertain scenarios.

\begin{figure}[h]
	\centering
	\includegraphics[width=1\columnwidth]{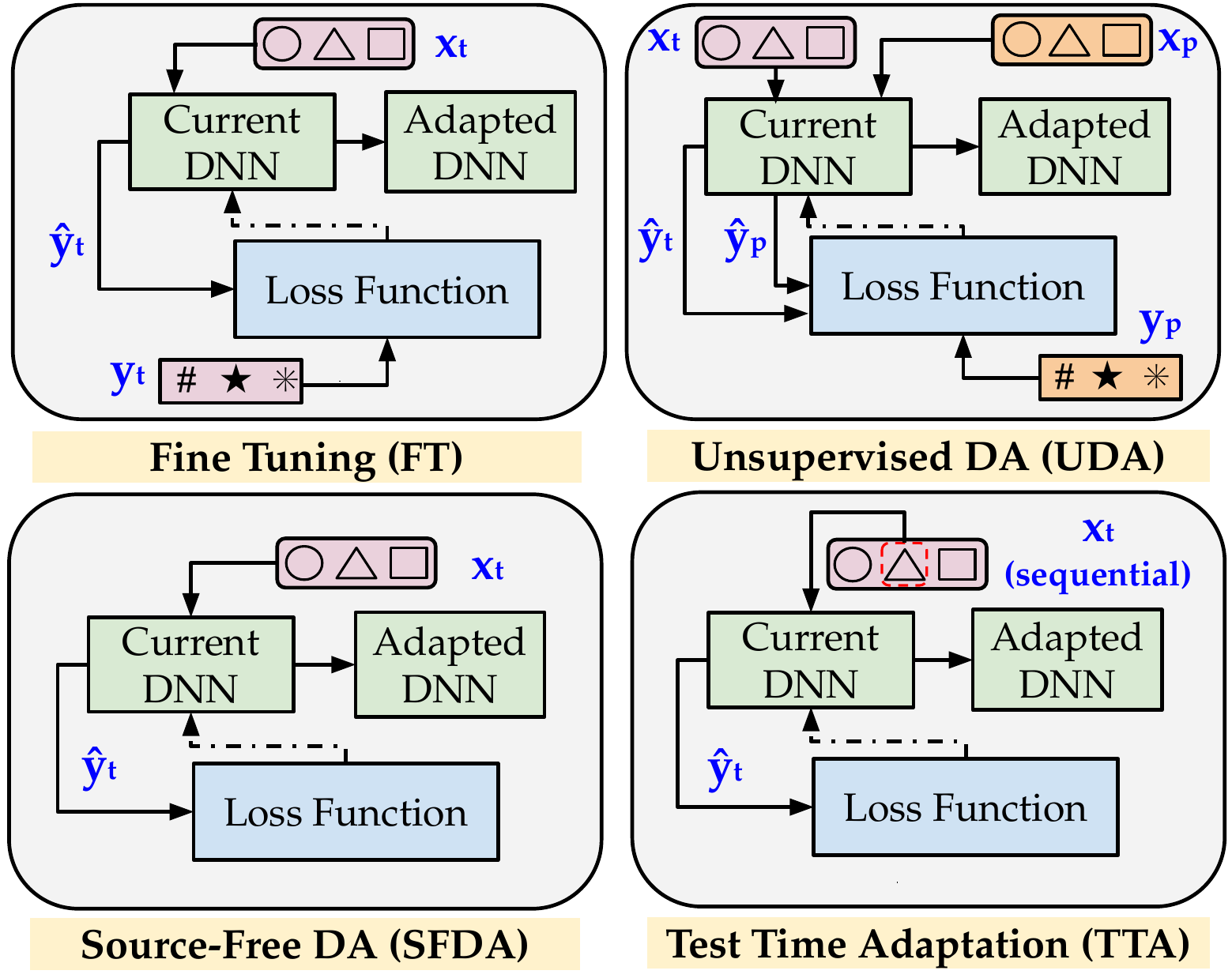} 
	\caption{Domain Generalization (DG) approaches. ($x_t$, $y_t$) indicate the current test sample and its true label, while ($x_p$, $y_p$) indicate  source sample and its true label. $\hat{y}_{t}$ indicates the corresponding predictions by the current \gls{dnn}. We point out that \FW is a TTA approach.}
	\label{fig:tta}

\end{figure}

\noindent \textbf{Source-Free Domain Adaptation (SFDA).}~Conversely from \gls{uda}, in \gls{sfda} the \gls{dnn} adaptation is performed using unlabeled data from the target domain only \cite{li2020model,kurmi2021domain}. While \gls{sfda} approaches take  into account numerous losses for several epochs during optimization, the key downside is that it can hardly be applied in real-time learning settings. \smallskip

\noindent \textbf{Test Time Adaptation (TTA).}~A practical approach to address distributional shifts in real-time settings is \gls{tta}, which utilizes only unlabeled test data (online) to adapt the \gls{dnn}. While existing \gls{tta} approaches improve performance, they are sensitive to the diversity of samples in incoming distributions. For example, even in cases of minor changes in brightness, approaches adapting the normalization layer based on entropy minimization such as \cite{wang2020tent,niu2022efficient} can experience a significant decrease in accuracy, dropping to less than 19\% of accuracy on CIFAR-10. Moreover, existing methods are not aware of domain changes, so they continuously update the \gls{dnn} even in the presence of no domain change. Notice that decoupling the corruption from the features relevant for classification is extremely challenging. Such continuous adaptation  leads to the risk of catastrophic forgetting. While state of the art work \cite{wang2022continual} uses stochastic restoration of parameters to the initial domain to tackle the issue, it needs 78.37\% more storage for ResNet56 architecture than our proposed approach.  

\section{Distribution and Label Shift}
Different corruptions lead mainly to a distribution shift in the input data, which is also widely known as a covariate shift. Distribution shift happens when the distribution of input data changes while the distribution of true labels remain unchanged. In a real-life adaptation of \glspl{dnn} at inference time, we usually have a batch of samples to work with that have both distribution and label shift ( due to correlation in labels in certain scenarios) at the same time. \cref{fig:prior & covrarite shift} illustrates this scenario with an example. This setting is challenging for most existing \gls{tta} algorithms, but is more practical. 
\begin{figure}[t]
	\centering
	\includegraphics[width=\columnwidth]{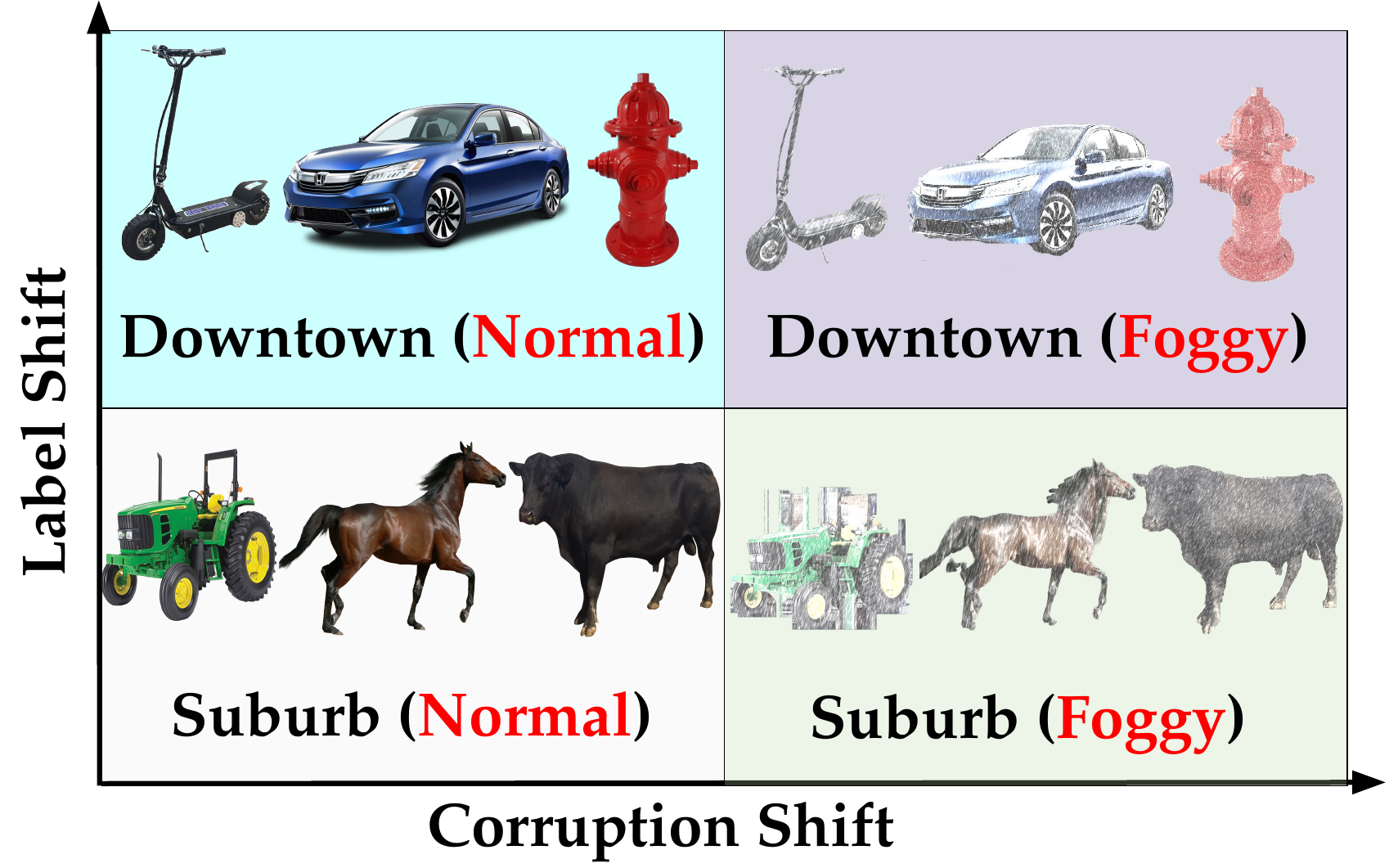} 
	\caption{Example of Label Distribution \& Corruption Shift.}
	\label{fig:prior & covrarite shift}

\end{figure}

\section{Hyperparameters and Implementation Details.}
We implement \FW with the PyTorch framework. For generating non \gls{iid} real time data flow, we adopted Dirichlet distribution (with parameter $\delta$) to create a non-\gls{iid} data flow. Furthermore, to simulate a domain change due to corruption, we feed samples from different test corruption types sequentially one after another following Dirichlet distribution with control parameter $\delta$, when all samples from the current test corruption are exhausted. With lower values of the Dirichlet parameter $\delta$, there is less diversity among online data batches, thus less correlation. The noise extractor is designed in a lightweight manner with three convolution layer with Leaky-ReLU non-linearity stacked sequentially. For the noise encoder, we use a sequential model consisting of two convolution units each consisting of a single convolution layer with ReLU non-linearity and a MaxPooling layer. The sequential model is followed by two dense layers. The sub-network encoder consists of two dense layers with ReLU non-linearity. For the reported results, we choose the output dimension of the sub-network encoder and the corruption encoder, i.e. the dimension of latent space to be 128 and the batch size is kept at 64, while  the size of the memory bank is kept the same as batch size. 

For CIFAR-100 corrupted dataset, we cannot insert samples from all classes in the memory bank. The parameter value $\delta=0.1$ is considered across all test scenarios unless otherwise specified. However,  we have found  empirically that presence of representative samples from the majority of the classes is sufficient.  We use Adam optimizer with learning rate $1 \times 10^3$ to perform the adaptation. To generate the sub-network signature, we first train a ResNet-56 backbone on the uncorrupted training dataset (CIFAR-10 and CIFAR-100) and fine-tune the sub-networks for 20 epochs using data from 15 train corruption domain to create the 15 sub-networks and their related signatures. For the hyper parameters, we assume a fixed set of values throughout the experiments, which are $\lambda_r =0.2$, $\lambda_e = 10$, $\varphi_{thresh}=0.005$ and momentum value $m=0.5$. As we make one step upgrade of the current \gls{bn} layer statistics by making sure that the samples of the memory banks are reliable, we chose a rather aggressive momentum value to weigh the normalization of current test samples highly.

\section{Power Measurement Setup}
As both Raspberry Pi and Jetson Nano do not have a system integrated in them to measure power at a certain instance, we use the setup of \cref{fig: power measurement} to calculate the energy consumption of different adaptation methods. The ina219 IC can provide accurate power consumption for a device at a certain time. We initially take some samples of power measurement to estimate the idle power usage of the device. Then for the whole adaptation period of each algorithm, we take samples of power drawn by the device every $10 ms$. Multiplying with the sampling time and averaging over batches of data we get a very good estimate of the energy consumption of different adaptation algorithms.  

\begin{figure}[h]
	\centering
	\includegraphics[width=\columnwidth]{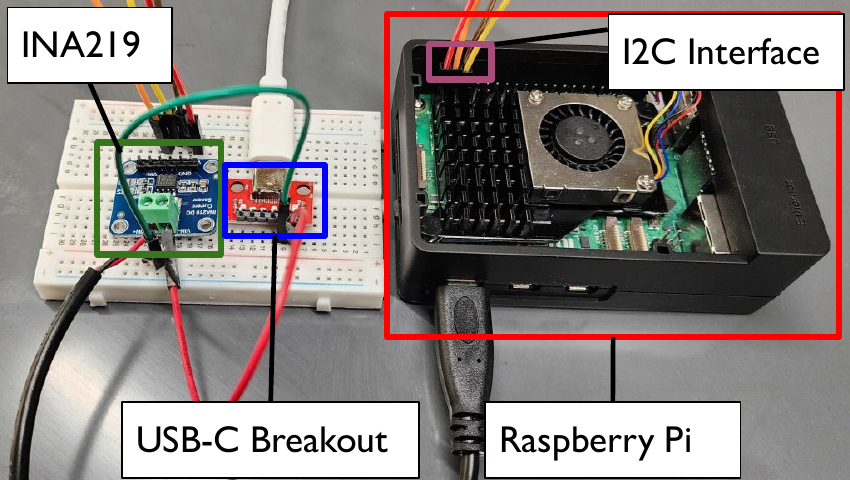} 
	\caption{Power Measurement Setup.}
	\label{fig: power measurement}

\end{figure}

\section{Performance of Corruption Extractor}

As discussed previously, in a corrupted data sample the corruption related features are intertwined with label information which impedes extraction of contextual information from data. 
\begin{algorithm} [t]
\caption{Memory Bank Construction}\label{alg:memory bank}
\begin{algorithmic}
\STATE {\bfseries Input:} A test sample $x^t$; and associated corruption embedding $C^t$;
current corruption projection $C^{curr}$

\STATE {\bfseries Define:} \hspace*{0.1cm} memory bank $\mathcal{M}$; total capacity $\mathcal{N}$;\\ class distribution $n[y]$, where $ y\in Y$; 
total occupancy $\mathcal{O}c$

\STATE Calculate $\hat{y} = \mbox{arg max}_y \,f_{\theta}^{'}(y\mid x^t)$
\IF {$n[\hat{y}] < \lceil \frac{\mathcal{N}}{\mid{Y}\mid}\rceil \,and \,\mathcal{O}c < \mid \mathcal{N} \mid$}

\STATE Add $(x^t,C^t)$ to $\mathcal{M}$ \\
\ELSE
    \STATE Calculate cosine similarity $f_{sim}$ of corruption \\projection among instances in $\left \{ n[\hat{y}] \cup x^t\right \}$ and $\mathcal{C}^{curr} $\\
    \STATE Find instance $(\tilde{x},\tilde{C})$ in $n[\hat{y}]$ with the lowest similarity $\mbox{arg min}_{x\in n[\hat{y}]} \,f_{sim}(x, C^{curr})$ to current signature \\
\ENDIF
\IF {$f_{sim}(\tilde{x}) > f_{sim}(x^t)$}
    \STATE {Discard $(x^t, C^t)$}
\ELSE
    \STATE {Remove instance $(\tilde{x},\tilde{C})$ from $\mathcal{M}$\\
    \STATE Add $(x^t, C^t)$ to $\mathcal{M}$}
\ENDIF
\end{algorithmic}
\end{algorithm}

To verify whether the corruption extractor is effective in extracting corruption information from data, we trained our corruption extractor and encoder using data available from 15 different  corruptions and evaluate how it performs for unknown corruption and different severity. \cref{fig:corruption tsne} shows the t-distributed stochastic neighbor embedding (t-SNE) of the projections in the latent space of data from different corruption domains and different levels of severity.

In \cref{fig: Encoding image(low severity)} and \cref{fig: Encoding image(high severity)}, the corruption projections are made directly using the input data. \textbf{Here, we can observe that directly encoding the corrupted data does not produce clusters with tight boundaries.} Interestingly, \cref{fig: Encoding corruption_residual(low severity)} and \cref{fig: Encoding corruption_residual(high severity)} point out that projecting the corruption signature extracted from the corrupted data does produce significantly better clusters. Moreover, although samples from the ``spatter'' and ``saturate'' have an overlapping boundary in latent space, from \cref{table:Noise Type Comparison} \textbf{we can see that although they are visually dissimilar, a subnetwork that performs well for the ``spatter'' also works well for ``saturate''}. This means that \FW is effective not only in categorizing and mapping corruptions, but also in designing appropriate subnetworks. Moreover, \cref{fig: Encoding corruption_residual(high severity)} shows that the corruption extractor produces an even better cluster with higher severity data from unknown corruption domain even if it was trained on a lower corruption severity. This proves the intuition that the corruption extractor extracts corruption information rather than simply over-fitting to the joint distribution of data and corruption. 

\renewcommand{\arraystretch}{0.2}
\begin{table*}[!ht]
\centering
\normalsize
%\begin{tabular}{ m{3cm} m{3cm} m{3cm} m{3cm} }
\begin{tabular}{ c c c c c }
%\cline{2-5}
%& Image Samples & Layer 13 & Layer 12 & Layer 11 & Layer 10 & Layer 9 & Layer 8 & Layer 7 & Layer 6 \\
%\hline
\toprule
\textbf{True Corruption Class} & \textbf{Spatter} & \textbf{Gaussian Blur} & \textbf{Speckle Noise} & \textbf{Saturate}
\\
\midrule
%\cline{2-5}
Data Sample &
\begin{minipage}{.15\textwidth}
    %\vspace{0.02\textwidth}
      \includegraphics[width=\linewidth, height=20mm]{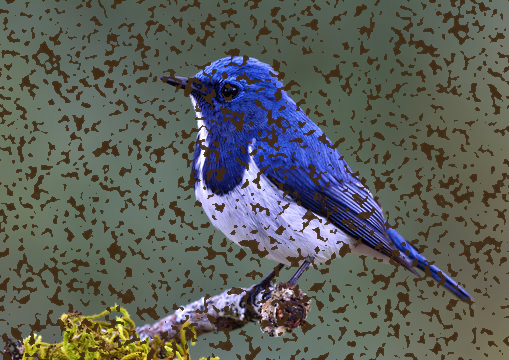}
\end{minipage}
&
\begin{minipage}{.15\textwidth}
    %\vspace{0.02\textwidth}
      \includegraphics[width=\linewidth, height=20mm]{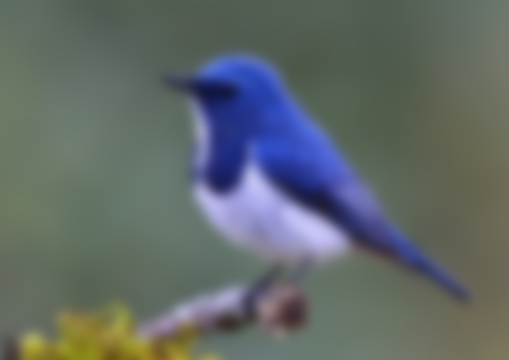}
\end{minipage}
&
\begin{minipage}{.15\textwidth}
    %\vspace{0.02\textwidth}
      \includegraphics[width=\linewidth, height=20mm]{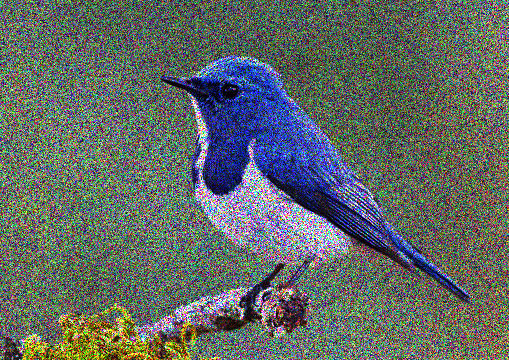}
\end{minipage}
&
\begin{minipage}{.15\textwidth}
    %\vspace{0.02\textwidth}
      \includegraphics[width=\linewidth, height=20mm]{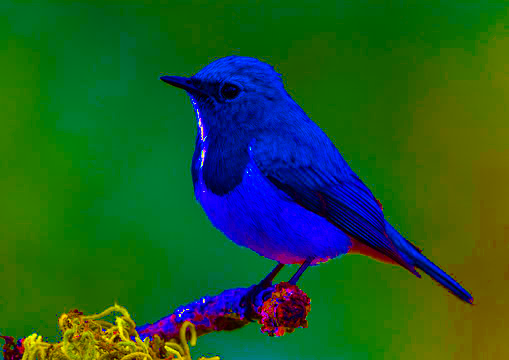}
\end{minipage}
\\
%\cline{2-5}
\midrule
JPEG Compression&84\% & 64.3\% & 78.1\% & 84.8\%
\\
%\cline{2-5}
\midrule
Glass Blur &72.5\% & 79.4\% & 56\% & 74\%
\\
%\cline{2-5}
\midrule
Shot Noise & 78\% & 39.4\% & 83.8\% & 75.8\%
\\
%\cline{2-5}
\midrule
Brightness & 83.7\% & 52.8\% & 66.7\% & 85.6\% \\
\bottomrule
\end{tabular}

\caption{Performance (accuracy on CIFAR-10) comparison of the sub-network signature closest to the unknown corruption signatures in the latent space. The left most column indicates the corresponding sub-network signatures. The unknown corruption domain and its closest sub-network signatures from latent space are: Spatter $ \rightarrow $ JPEG Compression, Gaussian Blur $\rightarrow$ Glass Blur, Speckle Noise $\rightarrow$ Shot Noise, Saturate $\rightarrow$ Brightness.}

\label{table:Noise Type Comparison}
\end{table*}

\section{Memory Bank Construction Process}
The memory bank construction process is described in detail in \cref{alg:memory bank}.

\begin{figure*} [t]
        \centering
        \begin{subfigure}[b]{.22\textwidth}
                \centering
                \includegraphics[width=\linewidth]{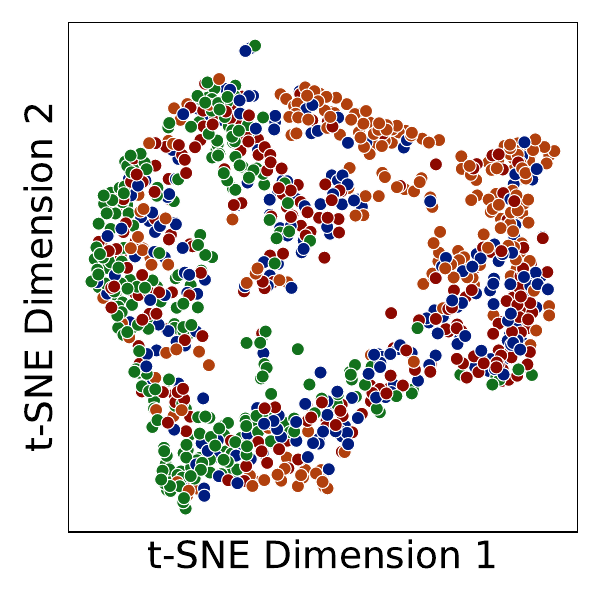}
                \setlength\abovecaptionskip{-0.5cm}
                \caption{Projection of Corrupted Data (Low Severity)}
                \label{fig: Encoding image(low severity)}
        \end{subfigure}
        \begin{subfigure}[b]{.22\textwidth}
                \centering
                \includegraphics[width=\linewidth]{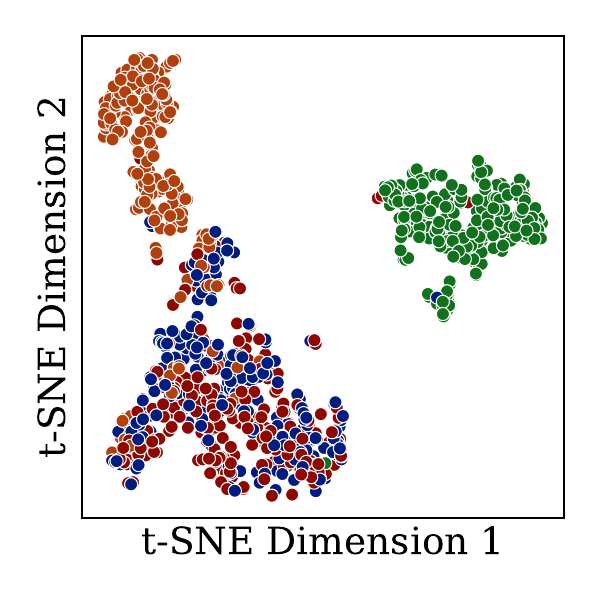}
                \setlength\abovecaptionskip{-0.5cm}
                \caption{Projection of Extracted Corruption (Low Severity)}
                \label{fig: Encoding corruption_residual(low severity)}
        \end{subfigure}
         \begin{subfigure}[b]{.22\textwidth}
                \centering
                \includegraphics[width=\linewidth]{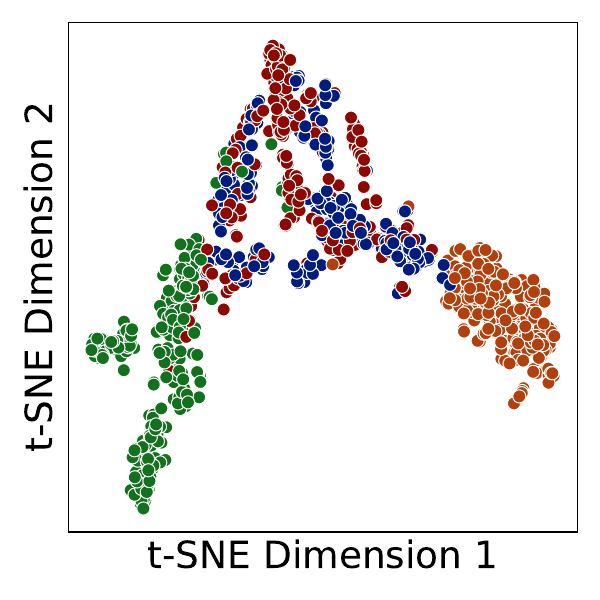}
                \setlength\abovecaptionskip{-0.5cm}
                \caption{Projection of Corrupted Data (High Severity)}
                \label{fig: Encoding image(high severity)}
        \end{subfigure}
         \begin{subfigure}[b]{.22\textwidth}
                \centering
                \includegraphics[width=\linewidth]{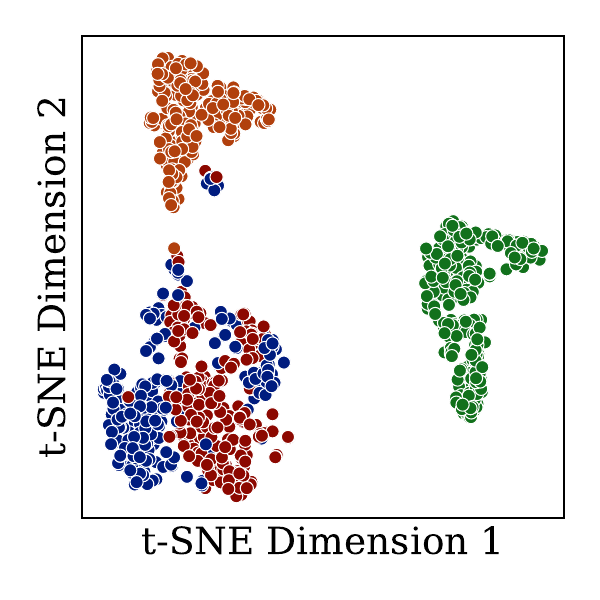}
                 \setlength\abovecaptionskip{-0.5cm}
                \caption{Projection of Extracted Corruption (High Severity)}
                \label{fig: Encoding corruption_residual(high severity)}
        \end{subfigure}
         \setlength\abovecaptionskip{-0.01cm}
        \caption{t-distributed stochastic neighbor embedding (t-SNE) of different samples in latent space for CIFAR-10. Here the green, orange, blue, and red colors indicate Gaussian blur, speckle noise, and saturate; respectively. }\label{fig:sample_performance}
        \label{fig:corruption tsne}

\end{figure*}
\section{More Details of Corruption Encoder}
The processes involved both in the training and inference phase of the proposed Corruption Encoder is illustrated through \cref{alg:corruption-encoder}.
\begin{algorithm}[tb]
\caption{Corruption Encoder} \label{alg:corruption-encoder}
\begin{algorithmic}
    
\STATE {\bfseries Input:} dataset $\mathcal{D}^d$; epochs $E$; batch size $N$;  constant $\tau$ and $\lambda_e$; embed dimension $o$; pair downsampler $\mathcal{G}(.)$ ;
transformation set $\mathcal{T}$; structure of $g(.)$ and $h(.)$

\STATE {\bfseries Output:} projection vector into corruption latent space \COMMENT{training} 

\FOR{$epoch = 1$ {\bfseries to} $E$}
\STATE sample $\left \{ x_i,D^i \right \}_{i=1}^{N}$ 
\STATE sample two augmentations $\mathcal{T}^a$, $\mathcal{T}^b \sim  \mathcal{T}$ \\
\STATE generate two down sampled data\\
\STATE $\mathbf{G}_1(x_i)$, $\mathbf{G}_2(x_i) = \mathcal{G}(x_i)$ \\
\STATE calculate $\tilde{\mathbf{G}_1}(x_i)$, $\tilde{\mathbf{G}_2}(x_i)$ using \cref{eqn:g(x)} \\
\STATE calculate $\mathcal{L}_n$ using \cref{eqn:Ln} \\
\STATE generate $x_i^a$, $x_i^b$ = $\mathcal{T}^a({x_i})$,$\mathcal{T}^b({x_i})$ \\
\STATE generate $\mathbf{G}_1(x_i^a)$, $\mathbf{G}_2(x_i^a) = \mathcal{G}(x_i^a)$ and \\
\STATE $\mathbf{G}_1(x_i^b)$, $\mathbf{G}_2(x_i^b) = \mathcal{G}(x_i^b)$ 
\STATE extract $x_i$'s corruption $x_{res_i}$, by concatenating \\ 
\STATE $x_{res_i}$ = $g((\mathbf{G}_1(x_i)) \:\: || \:\: g(\mathbf{G}_2(x_i))$ 
\STATE calculate the latent space projections $C^i$ by \\ $C^i = h(x_{{res}_i})$ 
\STATE calculate $\mathcal{L}_D$ using  \cref{eqn:supcon} \\
\STATE calculate overall loss using \cref{eq:training_corruption_encoder} \\
\STATE update $g(.)$ and $h(.)$ to minimize $\mathcal{L}$
\ENDFOR
\COMMENT{test}

\FOR{ x {\bfseries in} $\mathcal{D}^u$} 

\STATE generate $\mathbf{G}_1(x)$, $\mathbf{G}_2(x) = \mathcal{G}(x)$ \\
\STATE calculate and concatenate corruption features \\
\STATE  $x_{res}$ = $g(\mathbf{G}_1(x)) || g(\mathbf{G}_2(x))$\\
\STATE  calculate projection into latent space by \\
$C =h(x_{res})$
\ENDFOR
\end{algorithmic}
\end{algorithm}

\subsection{Theoretical Analysis of the Corruption Encoder}
We offer insights into the information learned by the corruption encoding process through a theoretical analysis of the corruption encoder for additive noise.
\begin{prop} \label{prop:prop_1}
Let, two noisy observation from original sample x be $y_1 = G_1(x)$ and $y_2 = G_2(x)$. Assuming zero mean and independent additive noise; $y_1 = x + e_1$ and $y_1 = x + e_1$ ; where noise is denoted by $e_i$. If $e_i$ can be approximated by $g(.): e_i = g_{\phi}(y_i)$ by minimizing MSE loss between clean observation $x_i$ and noisy observation $y_i$, minimizing the loss between two noisy observation approximate the same thing.
\end{prop}

\begin{proof}
\begin{align*}
    g(y_1,x ; \phi) &= \operatorname*{argmin}_\phi \mathbb{E} \left[ \left \| y_1 - g_{\phi}(y_1)-x \right \|_2^2 \right] \\
    &= \operatorname*{argmin}_\phi \mathbb{E}\left [ \left \| g_{\phi}(y_1) \right \|_2^2 -2y_1^Tg_{\phi}(y_1)+2x^Tg_{\phi}(y_1) \right ] \\
    g(y_1,y_2 ; \phi) &= \operatorname*{argmin}_\phi \mathbb{E} \left[ \left \| y_1 - g_{\phi}(y_1)-y_2 \right \|_2^2 \right] \\
    &= \operatorname*{argmin}_\phi \mathbb{E} \left[ \left \| y_1 - g_{\phi}(y_1)-x -e_2 \right \|_2^2 \right] \\
    &= \operatorname*{argmin}_\phi \mathbb{E} [ \left \| g_{\phi}(y_1) \right \|_2^2 -2y_1^Tg_{\phi}(y_1)+2x^Tg_{\phi}(y_1) \\
    &\quad+ 2e_2^Tg_{\phi}(y_1) ] \\ 
    &= \operatorname*{argmin}_\phi \mathbb{E}\left [ \left \| g_{\phi}(y_1) \right \|_2^2 -2y_1^Tg_{\phi}(y_1)+2x^Tg_{\phi}(y_1) \right ] \\
    &= g(y_1,x ; \phi)
\end{align*}
Assuming zero mean $\mathbf{E}(e_i)=0$ and independent noise the second to last equality is satisfied 
\end{proof}
\end{document}